\documentclass[11pt]{article}
\pdfoutput=1

\usepackage{amsmath,amssymb}
\usepackage{bm}
\usepackage{natbib}
\usepackage[usenames]{color}
\usepackage{amsthm}

\usepackage{multirow} 
\usepackage{enumerate}
\usepackage{bbm}

\allowdisplaybreaks

\usepackage[colorlinks,
linkcolor=red,
anchorcolor=blue,
citecolor=blue
]{hyperref}

\usepackage{setspace}
\usepackage[left=1in, right=1in, top=1in, bottom=1in]{geometry}

\usepackage{xcolor}
\usepackage{footnote}

\usepackage[utf8]{inputenc} 
\usepackage[T1]{fontenc}    
\usepackage{hyperref}       
\usepackage{url}            
\usepackage{booktabs}       
\usepackage{multirow}
\usepackage{amsfonts}       
\usepackage{nicefrac}       
\usepackage{microtype}      

\usepackage{graphicx}
\usepackage{amsthm}
\usepackage{amsmath}
\usepackage{amssymb}
\usepackage{enumerate} 
\usepackage{mathrsfs} 
\usepackage{listings} 
\usepackage{color} 
\definecolor{mygreen}{RGB}{28,172,0} 
\definecolor{mylilas}{RGB}{170,55,241}
\usepackage{caption}
\usepackage{subcaption} 
\usepackage{yfonts} 
\usepackage[boxed]{algorithm2e} 
\usepackage{bbm} 

\theoremstyle{definition}
\newtheorem{theorem}{Theorem}[section]

\newtheorem{lemma}[theorem]{Lemma}

\newtheorem{definition}[theorem]{Definition}
\newtheorem{remark}[theorem]{Remark}
\newtheorem{claim}[theorem]{Claim}
\newtheorem{fact}[theorem]{Fact}

\newtheorem{assumption}[theorem]{Assumption}
\theoremstyle{definition}

\newcommand{\N}{\mathbb{N}}

\newcommand{\R}{\mathbb{R}}

\newcommand{\E}{\mathbb{E}}

\let\P\BP

\let\hat\widehat

\newcommand{\f}{\frac}

\newcommand{\eps}{\varepsilon}
\renewcommand{\r}{\right}
\renewcommand{\l}{\left}

\newcommand{\pnorm}[2]{\left\|#1\right\|_{#2}}
\newcommand{\norm}[1]{\left\|#1\right\|}
\newcommand{\ip}[2]{\left\langle #1, #2 \right\rangle}

\newcommand{\ceil}[1]{\lceil #1 \rceil}

\newcommand{\ind}{{\mathbbm{1}}}

\newcommand{\summ}[2]{\sum_{#1 = 1}^{#2}}
\newcommand{\summm}[3]{\sum_{#1 = #2}^{#3}}

\newcommand{\calD}{\mathcal{D}}

\newcommand{\calF}{\mathcal{F}}

\newcommand{\calG}{\mathcal{G}}
\newcommand{\calH}{\mathcal{H}}

\DeclareMathOperator{\poly}{poly}

\newcommand{\iid}{\stackrel{\rm i.i.d.}{\sim}}

\makeatletter
\renewcommand*{\eqref}[1]{%
  \hyperref[{#1}]{\textup{\tagform@{\ref*{#1}}}}%
}
\makeatother

\newcommand{\mgtr}{\succ}

\newcommand{\opt}{\mathsf{OPT}}
\newcommand{\OPT}{\mathsf{OPT}}

\newcommand{\oF}{H}

\numberwithin{equation}{section}

\author
{
    Spencer Frei\thanks{Department of Statistics, University of California, Los Angeles, CA 90095, USA; e-mail: {\tt spencerfrei@ucla.edu}}
    ~~~and~~~
	Yuan Cao\thanks{Department of Computer Science, University of California, Los Angeles, CA 90095, USA; e-mail: {\tt yuancao@cs.ucla.edu}} 
	~~~and~~~
	Quanquan Gu\thanks{Department of Computer Science, University of California, Los Angeles, CA 90095, USA; e-mail: {\tt qgu@cs.ucla.edu}}
}
\date{}
\title{\huge Agnostic Learning of a Single Neuron with Gradient Descent}

\begin{document}
\maketitle
\begin{abstract}
    We consider the problem of learning the best-fitting single neuron as measured by the expected square loss $\E_{(x,y)\sim \mathcal{D}}[(\sigma(w^\top x)-y)^2]$ over some unknown joint distribution $\mathcal{D}$ by using gradient descent to minimize the empirical risk induced by a set of i.i.d. samples $S\sim \mathcal{D}^n$.   The activation function $\sigma$ is an arbitrary Lipschitz and non-decreasing function, making the optimization problem nonconvex and nonsmooth in general, and covers typical neural network activation functions and inverse link functions in the generalized linear model setting.   In the agnostic PAC learning setting, where no assumption on the relationship between the labels $y$ and the input $x$ is made, if the optimal population risk is $\mathsf{OPT}$, we show that gradient descent achieves population risk $O(\mathsf{OPT})+\eps$ in polynomial time and sample complexity when $\sigma$ is strictly increasing.  For the ReLU activation, our population risk guarantee is $O(\mathsf{OPT}^{1/2})+\eps$.  When labels take the form $y = \sigma(v^\top x) + \xi$ for zero-mean sub-Gaussian noise $\xi$, we show that the population risk guarantees for gradient descent improve to $\mathsf{OPT} + \eps$.  Our sample complexity and runtime guarantees are (almost) dimension independent, and when $\sigma$ is strictly increasing, require no distributional assumptions beyond boundedness.  For ReLU, we show the same results under a nondegeneracy assumption for the marginal distribution of the input.  
\end{abstract}
\section{Introduction}
We study learning the best possible single neuron that captures the relationship between the input $x\in \R^d$ and the output label $y\in \R$ as measured by the expected square loss over some unknown but fixed distribution $(x,y)\sim \calD$.  In particular, for a given activation function $\sigma:\R\to \R$, we define the population risk $F(w)$ associated with a set of weights $w$ as
\begin{equation} F(w) := (1/2) \E_{(x,y) \sim \calD} \l[ \l( \sigma(w^\top x) - y \r)^2\r].
\label{eq:optim.obj}
\end{equation}
 The activation function is assumed to be non-decreasing and Lipschitz, and includes nearly all activation functions used in neural networks such as the rectified linear unit (ReLU), sigmoid, $\tanh$, and so on.  In the agnostic PAC learning setting \citep{kearns.agnostic}, no structural assumption is made regarding the relationship of the input and the label, and so the best-fitting neuron could, in the worst case, have 
nontrivial population risk.  Concretely, if we denote
 \begin{equation}
    v := \mathrm{argmin}_{\pnorm w2 \leq 1} F(w),\quad  \opt:= F(v), 
    \label{def:v.pop.risk.minimizer}
\end{equation}
then the goal of a learning algorithm is to (efficiently) return weights $w$ such that the population risk $F(w)$ is close to the best possible risk $\OPT$.  The agnostic learning framework stands in contrast to the \textit{realizable} PAC learning setting, where one assumes $\opt=0$, so that there exists some $v$ such that the labels are given by $y=\sigma(v^\top x)$.   

The learning algorithm we consider in this paper is empirical risk minimization  using vanilla gradient descent.  We assume we have access to a set of i.i.d. samples $\{(x_i,y_i)\}_{i=1}^n\sim \calD^n$, and we run gradient descent with a fixed step size on the empirical risk $\hat F(w) = (1/2n)\textstyle \summ i n (\sigma(w^\top x_i)-y_i)^2$.   A number of early neural network studies pointed out that the landscape of the empirical risk of a single neuron has unfavorable properties, such as a large number of spurious local minima~\citep{brady1989,auer1995}, and led researchers to instead study gradient descent on a convex surrogate loss~\citep{helmbold95worstcase,helmbold99relativeloss}.  Despite this, we are able to show that gradient descent on the empirical risk itself finds weights that not only have small empirical risk but small population risk as well.

Surprisingly little is known about neural networks trained by minimizing the empirical risk with gradient descent in the agnostic PAC learning setting.  We are aware of two works \citep{allenzhu.3layer,allenzhu.kernel} in the \textit{improper} agnostic learning setting, where the goal is to return a hypothesis $h\in \calH$ that achieves population risk close to $\hat \opt$, where $\hat \opt$ is the smallest possible population risk achieved by a different set of hypotheses $\hat \calH$.  Another work considered the random features setting where only the final layer of the network is trained and the marginal distribution over $x$ is uniform on the unit sphere~\citep{vempala}. But none of these address the simplest possible neural network: that of a single neuron $x\mapsto \sigma(w^\top x)$. We believe a full characterization of what we can (or cannot) guarantee for gradient descent in the single neuron setting will help us understand what is possible in the more complicated deep neural network setting.  Indeed, two of the most common hurdles in the analysis of deep neural networks trained by gradient descent---nonconvexity and nonsmoothness---are also present in the case of the single neuron.  We hope that our analysis in this relatively simple setup will be suggestive of what is possible in more complicated neural network models. 

Our main contributions can be summarized as follows.
\begin{enumerate}[1)]
\item \textbf{Agnostic setting} (Theorem \ref{theorem:agnostic}). Without any assumptions on the relationship between $y$ and $x$, and assuming only boundedness of the marginal distributions of $x$ and $y$, we show that for any $\eps>0$, gradient descent finds a point $w_t$ with population risk $O(\opt) + \eps$ with sample complexity $O(\eps^{-2})$ and runtime $O(\eps^{-1})$ when $\sigma(\cdot)$ is strictly increasing and Lipschitz.  When $\sigma$ is ReLU, we obtain a population risk guarantee of $O(\opt^{1/2})+\eps$ with sample complexity $O(\eps^{-4})$ and runtime $O(\eps^{-2})$ when the marginal distribution of $x$ satisfies a nondegeneracy condition (Assumption \ref{assumption:marginal.spread}). The sample and runtime complexities are independent of the input dimension for both strictly increasing activations and ReLU.

\item \textbf{Noisy teacher network setting} (Theorem \ref{theorem:glm}). When $y = \sigma(v^\top x) + \xi$, where $\xi|x$ is zero-mean and sub-Gaussian (and possibly dependent on $x$), we demonstrate that gradient descent finds $w_t$ satisfying $F(w_t) \leq \opt + \eps$ for activation functions that are strictly increasing and Lipschitz assuming only boundedness of the marginal distribution over $x$.  The same result holds for ReLU under a marginal spread assumption (Assumption \ref{assumption:marginal.spread}).  The runtime and sample complexities are of order $\tilde O(\eps^{-2})$, with a logarithmic dependence on the input dimension.  When the noise is bounded, our guarantees are dimension independent.   If we further know $\xi \equiv 0$, i.e. the learning problem is in the realizable rather than agnostic setting, we can improve the runtime and sample complexity guarantees from $O(\eps^{-2})$ to $O(\eps^{-1})$ by using online stochastic gradient descent (Theorem \ref{theorem:gd.loss}).   
\end{enumerate}

\section{Related work}
Below, we provide a high-level summary of related work in the agnostic learning and teacher network settings.  Detailed comparisons with the most related works will appear after we present our main theorems in Sections \ref{sec:agnostic} and \ref{sec:noisy}.  In Appendix \ref{appendix:comparisons}, we provide tables that describe the assumptions and complexity guarantees of our work in comparison to related results. 

\noindent \textbf{Agnostic learning:}  The simplest version of the agnostic regression problem is to find a hypothesis that matches the performance of the best \textit{linear} predictor.  In our setting, this corresponds to $\sigma$ being the identity function.  This problem is completely characterized:~\citet{shamir15} showed that any algorithm that returns a linear predictor $v$ has risk $\opt + \Omega(\eps^{-2}\wedge d\eps^{-1})$ when the labels satisfy $|y|\leq 1$ and the marginal distribution over $x$ is supported on the unit ball, matching upper bounds proved by~\citet{srebro.mirror} using mirror descent.  

When $\sigma$ is not the identity, related works are scarce.~\citet{goel2017relupoly} studied agnostic learning of the ReLU on distributions supported on the unit sphere but had runtime and sample complexity exponential in $\eps^{-1}$. In another work on learning a single ReLU,~\citet{goel2019relugaussian} showed that learning up to risk $\opt+\eps$ in polynomial time is as hard as the problem of learning sparse parities with noise, long believed to be computationally intractable.  Additionally, they provided an approximation algorithm that could learn up to $O(\opt^{2/3})+\eps$ risk in $\poly(d, \eps^{-1})$ time and sample complexity when the marginal distribution over $x$ is a standard Gaussian.  
In a related but incomparable set of results in the improper agnostic learning setting,~\citet{allenzhu.3layer} and \citet{allenzhu.kernel} showed that multilayer ReLU networks trained by gradient descent can match the population risk achieved by multilayer networks with smooth activation functions.~\citet{vempala} studied agnostic learning of a one-hidden-layer neural network when the first layer is fixed at its (random) initial values and the second layer is trained. A very recent work by \citet{diakonikolas2020approximation} showed that population risk $O(\opt)+\eps$ can be achieved for the single ReLU neuron by appealing to gradient descent on a convex surrogate for the empirical risk.

\noindent\textbf{Teacher network:} The literature refers to the case of $y= \sigma(v^\top x) + \xi$ for some possible  zero mean noise $\xi$ variously as the ``noisy teacher network'' or ``generalized linear model'' (GLM) setting, and is related to the probabilistic concepts model~\citep{kearns.probabilistic}.  In the GLM setting, $\sigma$ plays the role of the inverse link function; in the case of logistic regression, $\sigma$ is the sigmoid function.   

The results in the teacher network setting can be broadly characterized by (1) whether they cover arbitrary distributions over $x$ and (2) the presence of noise (or lackthereof).  The GLMTron algorithm proposed by~\citet{kakade2011}, itself a modification of the Isotron algorithm of~\citet{kalai2009isotron}, is known to learn a noisy teacher network up to risk $\opt+\eps$ for any Lipschitz and non-decreasing $\sigma$ and any distribution with bounded marginals over $x$.~\citet{mei2016landscape} showed that gradient descent learns the noisy teacher network under a smoothness assumption on the activation function for a large class of distributions.~\citet{foster2018} provided a meta-algorithm for translating $\eps$-stationary points of the empirical risk to points of small population risk in the noisy teacher network setting.  A recent work by~\citet{mukherjee} develops a modified SGD algorithm for learning a ReLU with bounded adversarial noise on distributions where the input is bounded.  

Of course, any guarantee that holds for a neural network with a single fully connected hidden layer of arbitrary width holds for the single neuron, so in this sense our work can be connected to a larger body of work on the analysis of gradient descent used for learning neural networks.  The majority of such works are restricted to particular input distributions, whether it is Gaussian or uniform distributions~\citep{soltanolkotabi2017relus,tian2017relu,soltanolkotabi2019theoretical,zhanggu2019,goel.convotron,cao2019cnn}.~\citet{du2017} showed that in the noiseless (a.k.a. realizable) setting, a single neuron can be learned with SGD if the input distribution satisfies a certain subspace eigenvalue property.~\citet{yehudai20} studied the properties of learning a single neuron for a variety of increasing and Lipschitz activation functions using gradient descent, as we do in this paper, although their analysis was restricted to the noiseless setting.

\section{Agnostic learning setting}
\label{sec:agnostic}
We begin our analysis by assuming there is no \textit{a priori} relationship between $x$ and $y$,  so the population risk $\opt$ of the population risk minimizer $v$ defined in \eqref{def:v.pop.risk.minimizer} may, in general, be a large quantity.  If $\opt =0$, then $\sigma(v^\top x) = y$ a.s. and the problem is in the realizable PAC learning setting.  In this case, we can use a modified proof technique to get stronger guarantees for the population risk; see Appendix \ref{appendix:realizable} for the complete theorems and proofs in this setting.  We will thus assume without loss of generality that $0 < \opt \leq 1$. 

The gradient descent method we use in this paper is as follows.  We assume we have a training sample $\{(x_i,y_i)\}_{i=1}^n\iid \calD^n$, and define the empirical risk for weight $w$ by
\[ \hat F(w) = (1/2n)\textstyle \summ i n (\sigma(w^\top x_i) - y_i)^2.\]
We perform full-batch gradient updates on the empirical risk using a fixed step size $\eta$,
\begin{equation}
    w_{t+1} = w_t - \eta \nabla \hat F(w_t) = w_t - (\eta/n) \textstyle \summ i n (\sigma(w_t^\top x_i) - y_i) \sigma'(w_t^\top x_i) x_i,
    \label{eq:gd.updates}
\end{equation}
where $\sigma'(\cdot)$ is the derivative of $\sigma(\cdot)$.  If $\sigma$ is not differentiable at a point $z$, we will use its subderivative.

We begin by describing one set of activation functions under consideration in this paper.  
 \begin{assumption}
\begin{enumerate}[(a)]
\item $\sigma$ is continuous, non-decreasing, and differentiable almost everywhere. 
\item For any $\rho > 0$, there exists $\gamma >0$ such that  $\inf_{|z| \leq \rho} \sigma'(z) \geq \gamma > 0$. 
If $\sigma$ is not differentiable at $z\in[-\rho,\rho]$, assume that every subderivative $g$ on the interval satisfies $g(z)\geq \gamma$. 
\item $\sigma$ is $L$-Lipschitz, i.e. $|\sigma(z_1)-\sigma(z_2)|\leq L|z_1-z_2|$ for all $z_1,z_2$.
\end{enumerate}
\label{assumption:activation.fcn}
\end{assumption}
We note that if $\sigma$ is strictly increasing and continuous, then $\sigma$ satisfies Assumption \ref{assumption:activation.fcn}(b) since its derivative is never zero.  In particular, the assumption covers the typical activation functions in neural networks like leaky ReLU, softplus, sigmoid, tanh, etc., but excludes ReLU.  \citet{yehudai20} recently showed that when $\sigma$ is ReLU, there exists a distribution $\calD$ supported on the unit ball and unit length target neuron $v$ such that \textit{even in the realizable case} of $y = \sigma(v^\top x)$, if the weights are initialized randomly using a product distribution,  there exists a constant $c_0$ such that with high probability, $F(w_t) \geq c_0 >0$ throughout the trajectory of gradient descent.  This suggests that gradient-based methods for learning ReLUs are likely to fail without additional assumptions.  Because of this, they introduced the following marginal spread assumption to handle the learning of ReLU. 
\begin{assumption}
There exist constants $\alpha, \beta >0$ such that the following holds.  For any $w\neq u$, denote by $\calD_{w,u}$ the marginal distribution of $\calD$ on $\mathrm{span}(w,u)$, viewed as a distribution over $\R^2$, and let $p_{w,u}$ be its density function.  Then $\inf_{z\in \R^2:\norm{z}\leq \alpha} p_{w,u}(z) \geq \beta$. 
\label{assumption:marginal.spread}
\end{assumption}
This assumption covers, for instance, log-concave distributions like the Gaussian and uniform distribution with $\alpha, \beta = O(1)$~\citep{lovasz}.  We note that a similar assumption was used in recent work on learning halfspaces with Massart noise~\citep{diakonikolas2020}.  We will use this assumption for all of our results when $\sigma$ is ReLU.  Additionally, although the ReLU is not differentiable at the origin, we will denote by $\sigma'(0)$ its subderivative, with the convention that $\sigma'(0)=1$.  Such a convention is consistent with the implementation of ReLUs in modern deep learning software packages.

With the above in hand, we can describe our main theorem.  
\begin{theorem}\label{theorem:agnostic}
 Suppose the marginals of $\calD$ satisfy $\pnorm{x}2\leq B_X$ a.s. and $|y|\leq B_Y$ a.s. Let $a:=(|\sigma(B_X)|+B_Y)^2$.  When $\sigma$ satisfies Assumption \ref{assumption:activation.fcn}, let $\gamma>0$ be the constant corresponding to $\rho=2B_X$ and fix a step size $\eta \leq (1/8)\gamma L^{-3} B_X^{-2}$.  For any $\delta>0$, with probability at least $1-\delta$, gradient descent initialized at the origin and run for $T = \ceil{\eta^{-1}\gamma^{-1} L^{-1} B_X^{-1} [\opt + a n^{-1/2} \log^{1/2}(4/\delta)]^{-1}}$ iterations finds weights $w_t$, $t<T$, such that
 \begin{equation}
     F(w_t) \leq C_1 \opt + C_2 n^{-1/2},
     \label{eq:agnostic.F.bound}
     \end{equation}
     where $C_1 = 12\gamma^{-3} L^3 + 2$ and $C_2 = O(L^3 B_X^2\sqrt{\log(1/\delta)} + C_1 a \sqrt{\log(1/\delta)})$.  
     
     When $\sigma$ is ReLU, further assume that $\calD_x$ satisfies Assumption \ref{assumption:marginal.spread} for constants $\alpha, \beta >0$, and let $\nu = \alpha^4 \beta / 8 \sqrt 2$.  Fix a step size $\eta \leq (1/4) B_X^{-2}$. For any $\delta>0$, with probability at least $1-\delta$, gradient descent initialized at the origin and run for $T = \ceil{\eta^{-1}B_X^{-1}[\opt +an^{-1/2}\log^{1/2}(4/\delta)]^{-1/2}}$ iterations finds a point $w_t$ such that
     \begin{equation}
     F(w_t) \leq C_1 \opt^{1/2} + C_2 n^{-1/4}+ C_3 n^{-1/2},
     \label{eq:agnostic.F.bound}
     \end{equation}
     where $C_1 = O(B_X \nu^{-1})$, $C_2 = O( C_1 a^{1/2} \log^{1/4} (1/\delta))$, and $C_3 = O(B_X^2 \nu^{-1} \log^{1/2}(1/\delta))$. 
\end{theorem}
We remind the reader that the optimization problem for the empirical risk is highly nonconvex~\citep{auer1995} and thus any guarantee for the empirical risk, let alone the population risk, is nontrivial.  This makes us unsure if the suboptimal guarantee of $O(\opt^{1/2})$ for ReLU is an artifact of our analysis or a necessary consequence of nonconvexity.

In comparison to recent work,~\citet{goel2019relugaussian} considered the agnostic setting for the ReLU activation when the marginal distribution over $x$ is a standard Gaussian and showed that learning up to risk $\opt+\eps$ is as hard as learning sparse parities with noise.   By using an approximation algorithm of~\citet{awasthi}, they were able to show that one can learn up to $O(\opt^{2/3})+\eps$ with $O(\mathrm{poly}(d, \eps^{-1}))$ runtime and sample complexity.  In a very recent work, \citet{diakonikolas2020approximation} improved the population risk guarantee for the ReLU to $O(\opt) + \eps$ when the features are sampled from an isotropic log-concave distribution by analyzing gradient descent on a convex surrogate loss.  Projected gradient descent on this surrogate loss produces the weight updates of the GLMTron algorithm of~\citet{kakade2011}.  Using the solution found by gradient descent on the surrogate loss, they proposed an improper learning algorithm that improves the population risk guarantee from $O(\opt)+\eps$ to $(1+\delta) \OPT + \eps$ for any $\delta>0$.

By contrast, we show that gradient descent on the empirical risk learns up to a population risk of $O(\opt)+\eps$ for \textit{any} joint distribution with bounded marginals when $\sigma$ is strictly increasing and Lipschitz, even though the optimization problem is nonconvex.  In the case of ReLU, our guarantee holds for the class of bounded distributions over $x$ that satisfy the marginal spread condition of Assumption \ref{assumption:marginal.spread} and hence covers (bounded) log-concave distributions, although the guarantee is $O(\opt^{1/2})$ in this case.  For all activation functions we consider, the runtime and sample complexity guarantees do not have (explicit) dependence on the dimension.\footnote{We note that for some distributions, the $B_X$ term may hide an implicit dependence on $d$; more detailed comments on this are given in Appendix \ref{appendix:comparisons}.}  Moreover, we shall see in the next section that if the data is known to come from a noisy teacher network, the guarantees of gradient descent improve to $\opt+\eps$ for both strictly increasing activations and ReLU.  

In the remainder of this section we will prove Theorem \ref{theorem:agnostic}.  Our proof relies upon the following auxiliary errors for the true risk $F$:
\begin{align} 
\nonumber
G(w) &:= (1/2) \E_{(x,y)\sim \calD} \l[ \l( \sigma(w^\top x) - \sigma(v^\top x) \r)^2 \r],\\
\label{def:auxiliary.loss}
H(w) &:= (1/2)\E_{(x,y)\sim \calD} \l[ \l( \sigma(w^\top x) - \sigma(v^\top x) \r)^2 \sigma'(w^\top x) \r].
\end{align}
We will denote the corresponding empirical risks by $\hat G(w)$ and $\hat H(w)$.  We first note that $G$ trivially upper bounds $F$: this follows by a simple application of Young's inequality and, when $\E[y|x]~=~\sigma(v^\top x)$, by using iterated expectations.
 \begin{claim} \label{claim:Gbound:implies:Fbound}
 For any joint distribution $\calD$, for any vector $u$, and any continuous activation function $\sigma$, $F(u) \leq 2 G(u) + 2 F(v)$.  
 If additionally we know that $\E[y|x] = \sigma(v^\top x)$, we have $F(u)~=~G(u)~+~F(v)$.
 \end{claim}
 
 This claim shows that in order to show the population risk is small, it suffices to show that $G$ is small.  It is easy to see that if $\inf_{z\in \R} \sigma'(z) \geq \gamma > 0$, then $H(w)\leq \eps$ implies $G(w) \leq \gamma^{-1} \eps$, but the only typical activation function that satisfies this condition is the leaky ReLU.   Fortunately, when $\sigma$ satisfies Assumption \ref{assumption:activation.fcn}, or when $\sigma$ is ReLU and $\calD$ satisfies Assumption \ref{assumption:marginal.spread}, Lemma \ref{lemma:Hsurrogate} below shows that $H$ is still an upper bound for $G$.  The proof is deferred to Appendix~\ref{appendix:Hsurrogate}. 

\begin{lemma} \label{lemma:Hsurrogate} \label{lemma:relu.of.implies.f}
If $\sigma$ satisfies Assumption \ref{assumption:activation.fcn}, $\pnorm{x}2\leq B$ a.s., and $\pnorm{w}2\leq W$, then for $\gamma$ corresponding to $\rho = W B$, $H(w)\leq \eps$ implies $G(w)\leq \gamma^{-1} \eps$. 
If $\sigma$ is ReLU and $\calD$ satisfies Assumption \ref{assumption:marginal.spread} for some constants $\alpha, \beta >0$, and if for some $\eps>0$ the bound $\oF(w)\leq \beta \alpha^4 \eps / 8 \sqrt 2$ holds, then $\pnorm{w-v}2 \leq 1$ implies $G(w) \leq \eps$.
\end{lemma}
Claim \ref{claim:Gbound:implies:Fbound} and Lemma \ref{lemma:Hsurrogate} together imply that if gradient descent finds a point with auxiliary error $H(w_t) \leq O(\opt^\alpha)$ for some $\alpha \leq 1$, then gradient descent achieves population risk $O(\opt^\alpha)$.  In the remainder of this section, we will show that this is indeed the case.  In Section \ref{sec:strictly.increasing.activation}, we first consider activations satisfying Assumption \ref{assumption:activation.fcn}, for which we are able to show $H(w_t) \leq O(\opt)$.  In Section \ref{sec:relu.activation}, we show $H(w_t)\leq O(\opt^{1/2})$ for the ReLU.
\subsection{Strictly increasing activations}\label{sec:strictly.increasing.activation}
In Lemma \ref{lemma:agnostic.key.strictly.increasing} below, we show that $\hat H(w_t)$ is a natural quantity of the gradient descent algorithm that in a sense tells us how good a direction the gradient is pointing at time $t$, and that $\hat H(w_t)$ can be as small as $O(\hat F(v))$.  Our proof technique is similar to that of~\citet{kakade2011}, who studied the GLMTron algorithm in the (non-agnostic) noisy teacher network setup.
\begin{lemma}\label{lemma:agnostic.key.strictly.increasing}
Suppose that $\pnorm{x}2 \leq B_X$ a.s. under $\calD_x$.  Suppose $\sigma$ satisfies Assumption \ref{assumption:activation.fcn}, and let $\gamma$ be the constant corresponding to $\rho=2B_X$.  Assume $\hat F(v)>0$.   Gradient descent with fixed step size $\eta \leq (1/8) \gamma L^{-3} B_X^{-2}$ initialized at $w_0=0$ finds weights $w_t$ satisfying $\hat H(w_t) \leq 6 L^3\gamma^{-2} \hat F(v)$ within $T = \ceil{ \eta^{-1} \gamma^{-1} L^{-1} B_X^{-1} \hat F(v)^{-1}}$  iterations, with $\pnorm {w_t-v}2 \leq 1$ for each $t=0, \dots, T-1$.
\end{lemma}
Before beginning the proof, we first note the following fact, which allows us to connect terms that appear in the gradient to the square loss.  
\begin{fact}\label{fact:sigma.strictly.increasing} 
If $\sigma$ is strictly increasing on an interval $[a,b]$ with $\sigma'(z) \geq \gamma>0$ for all $z\in [a,b]$, and if $z_1,z_2\in [a,b]$, then, it holds that \begin{equation}
\gamma (z_1 - z_2)^2 \leq \l( \sigma(z_1 ) - \sigma(z_2) \r) (z_1 -z_2).
\label{eq:lb.identity.zs}
\end{equation}\end{fact}
\begin{proof}[Proof of Lemma \ref{lemma:agnostic.key.strictly.increasing}]
The proof comes from the following induction statement.  We claim that for every $t\in \N$, either (a) $\hat H(w_\tau) \leq 6 L^3 \gamma^{-2} \hat F(v)$ for some $\tau < t$, or (b) $\pnorm{w_t-v}2^2 \leq \pnorm{w_{t-1}-v}2^2 - \eta L\hat F(v)$ holds.  If this claim is true, then until gradient descent finds a point where $\hat H(w_t) \leq 6 L^3 \gamma^{-2} \hat F(v)$, the squared distance $\pnorm{w_t-v}2^2$ decreases by $\eta L \hat F(v)$ at every iteration.  Since $\pnorm{w_0-v}2^2 = 1$, this means there can be at most $1/(\eta L\hat F(v)) = \eta^{-1} L^{-1} \hat F(v)^{-1}$ iterations until we reach $\hat H(w_t) \leq 6 L^3 \gamma^{-2} \hat F(v)$.  

So let us now suppose the induction hypothesis holds for $t$, and consider the case $t+1$.  If (a) holds, then we are done.  So now consider the case that for every $\tau \leq t$, we have $\hat H(w_\tau) > 6 L^3 \gamma^{-2} \hat F(v)$.  
 Since (a) does not hold, $\pnorm{w_\tau -v}2^2 \leq \pnorm{w_{\tau-1}-v}2^2 - \eta L \hat F(v)$ holds for each $\tau=1, \dots, t$, and so $\pnorm{w_0-v}2=1$ implies
 \begin{equation}
     \pnorm{w_\tau - v}2 \leq 1\ \forall \tau \leq t.
     \label{eq:bounded.weights.agnostic}
 \end{equation}
In particular, $\pnorm {w_\tau}2 \leq 1 + \pnorm{v}2 \leq 2$ holds for all $\tau\leq t$.  By Cauchy--Schwarz, this implies $|w_\tau^\top x|\vee |v^\top x| \leq 2 B_X$ a.s.  By defining $\rho = 2 B_X$ and letting $\gamma$ be the constant from Assumption \ref{assumption:activation.fcn}, this implies $\sigma'(z) \geq \gamma>0$ for all $|z|\leq 2 B_X$.  Fact \ref{fact:sigma.strictly.increasing} therefore implies
\begin{equation}
    \sigma'(w_\tau^\top x) \geq \gamma > 0 \quad \text {   and   } \quad (\sigma(w_\tau^\top x) - \sigma(v^\top x))\cdot(w_\tau^\top x - v^\top x) \geq \gamma (w_\tau^\top x- v^\top x)^2 \quad \forall \tau \leq t.
    \label{eq:lb.identity}
\end{equation}

We proceed with the proof by demonstrating an appropriate lower bound for the quantity
\[ \pnorm{w_t-v}2^2 - \pnorm{w_{t+1}-v}2^2 = 2\eta \ip{\nabla \hat F(w_t)}{w_t-v} - \eta^2 \pnorm{\nabla \hat F(w_t)}2^2.\]
We begin with the inner product term.  We have
     \begin{align}
     \nonumber
         \big \langle \nabla \hat F(w_t) , w_t-v\big \rangle &= 
          (1/n) \summ i n \l( \sigma(w_t^\top x_i) - \sigma(v^\top x_i) \r) \sigma'(w_t^\top x_i) (w_t^\top x_i - v^\top x_i) \\
          \nonumber
         &\quad +  (1/n)  \summ i n \l( \sigma(v^\top x_i) - y_i \r) \gamma^{-1/2} \cdot \gamma^{1/2} \sigma'(w_t^\top x_i)(w_t^\top x_i -v^\top x_i)\\
         \nonumber
         &\geq (\gamma/n) \summ i n \l( w_t^\top x_i - v^\top x_i \r)^2 \sigma'(w_t^\top x_i)\\
         \nonumber
        &\quad - \f {\gamma^{-1}}{2n} \summ i n \l( \sigma(v^\top x_i) - y_i\r)^2 \sigma'(w_t^\top x_i) - \f \gamma {2n} \summ i n \l( w_t^\top x_i - v^\top x_i\r)^2 \sigma'(w_t^\top x_i)\\
        \nonumber
        &\geq \f \gamma 2 \summ i n (w_t^\top x_i-v^\top x_i)^2 \sigma'(w_t^\top x_i) - L \gamma^{-1} \hat F(v)\\
        &\geq \gamma L^{-2} \hat H(w_t) - L \gamma^{-1} \hat F(v).
        \label{eq:ip.lb.strictly.increasing}
     \end{align}
In the first inequality we used \eqref{eq:lb.identity} for the first term and Young's inequality for the second (and that $\sigma'\geq 0$).  For the final two inequalities, we use that $\sigma$ is $L$-Lipschitz.

For the gradient upper bound,
\begin{align}
\nonumber
\norm{\nabla \hat F(w)}^2 &\leq 2 \norm{\f 1n \summ i n (\sigma(w^\top x_i) - \sigma(v^\top x_i)) \sigma'(w^\top x_i) x_i}^2 \\
\nonumber
&\quad + 2 \norm{\f 1 n \summ i n (\sigma(v^\top x_i) - y_i) \sigma'(w^\top x_i) x_i}^2 \\
\nonumber
&\leq \f 2n \summ i n (\sigma(w^\top x_i) - \sigma(v^\top x_i))^2 \sigma'(w^\top x_i)^2 \pnorm{x_i}2^2 \\
\nonumber
&\quad + \f 2 n \summ i n (\sigma(v^\top x_i) - y_i)^2 \sigma'(w^\top x_i)^2 \pnorm{ x_i}2^2 \\
\nonumber
&\leq \f {2L B_X^2} n \summ i n  (\sigma(w^\top x_i)- \sigma(v^\top x_i))^2\sigma'(w^\top x_i) + 4L^2 B_X^2 \hat F(v)\\
&= 4L B_X^2 \hat H(w) + 4L^2 B_X^2 \hat F(v).
\label{eq:grad.ub.strictly.increasing}
\end{align}
The first inequality is due to Young's inequality, and the second is due to Jensen's inequality.  The last inequality holds because $\sigma$ is $L$-Lipschitz and $\pnorm{x}2\leq B_X$ a.s.  Putting \eqref{eq:ip.lb.strictly.increasing} and \eqref{eq:grad.ub.strictly.increasing} together and taking $\eta \leq (1/8) L^{-3} B_X^{-2} \gamma$, 
\begin{align*}
\norm{w_t-v}^2 - \norm{w_{t+1}-v}^2 &\geq 2 \eta ( \gamma L^{-2} \hat H(w_t) - L\gamma^{-1} \hat F(v)) - 4\eta^2 (L B_X^2 \hat H(w_t) + L^2 B_X^2 \hat F(v))\\
&\geq 2 \eta \l( \f {\gamma L^{-2}} 2 \hat H(w_t) - \f 5 2 L \gamma^{-1} \hat F(v)\r)\\
&\geq \eta \gamma L \hat F(v).
\end{align*}
The last inequality uses the induction assumption that $\hat H(w_t) \geq 6 L^3 \gamma^{-2} \hat F(v)$, completing the proof.
\end{proof}
Since the auxiliary error $\hat H(w)$ is controlled by $\hat F(v)$, we need to bound $\hat F(v)$.  When the marginals of $\calD$ are bounded, Lemma \ref{lemma:F(v).opt.concentration} below shows that $\hat F(v)$ concentrates around $F(v)=\opt$ at rate $n^{-1/2}$ by Hoeffding's inequality; for completeness, the proof is given in Appendix \ref{appendix:simple.proofs}.  
\begin{lemma}\label{lemma:F(v).opt.concentration}
If $\pnorm{x}2 \leq B_X$ and $|y|\leq B_Y$ a.s. under $\calD_x$ and $\calD_y$ respectively, and if $\sigma$ is non-decreasing, then for $a := \l( |\sigma(B_X)| + B_Y\r)^2$ and $\pnorm{v}2\leq 1$, we have with probability at least $1-\delta$,
\begin{align*}
    |\hat F(v) - \opt| \leq 3a \sqrt{n^{-1} \log(2/\delta)}.
\end{align*} 
\end{lemma}
The final ingredient to the proof is translating the bounds for the empirical risk to one for the population risk.  Since $\calD_x$ is bounded and we showed in Lemma \ref{lemma:agnostic.key.strictly.increasing} that $\pnorm{w_t-v}2\leq 1$ throughout the gradient descent trajectory, we can use standard properties of Rademacher complexity to get it done.  The proof for Lemma \ref{lemma:rademacher.complexity} can be found in Appendix \ref{appendix:simple.proofs}. 
 \begin{lemma}\label{lemma:rademacher.complexity}
 Suppose $\sigma$ is $L$-Lipschitz and $\pnorm{x}2\leq B_X$ a.s.  Denote $\ell(w; x) := (1/2) \l( \sigma(w^\top x) - \sigma(v^\top x)\r)^2$. For a training set $S\sim \calD^n$, let $\mathfrak{R}_S(\calG)$ denote the empirical Rademacher complexity of the following function class
 \[ \calG := \{ x\mapsto w^\top x : \pnorm{w-v}2\leq 1, \ \pnorm v 2 = 1 \}. \]
 Then we have
 \[ \mathfrak R(\ell \circ \sigma \circ \calG) = \E_{S\sim \calD^n} \mathfrak R_S(\ell \circ \sigma \circ \calG) \leq  2L^3 B_X^2/\sqrt n.\]
 \end{lemma}
 With Lemmas \ref{lemma:agnostic.key.strictly.increasing}, \ref{lemma:F(v).opt.concentration} and \ref{lemma:rademacher.complexity} in hand, the bound for the population risk follows in a straightforward manner. 

\begin{proof}[Proof of Theorem \ref{theorem:agnostic} for strictly increasing activations.]
By Lemma \ref{lemma:agnostic.key.strictly.increasing}, there exists some $w_t$ with $t<T$ and $\pnorm{w_t-v}2\leq 1$ such that $\hat H(w_t) \leq 6L^3 \gamma^{-2} \hat F(v)$.  By Lemmas \ref{lemma:Hsurrogate} and Lemma \ref{lemma:F(v).opt.concentration}, this implies that with probability at least $1-\delta/2$,
\begin{equation}
    \hat G(w_t) \leq 6 L^3 \gamma^{-3} \l( \opt + 3a n^{-1/2} \log^{1/2}(4/\delta)\r).\label{eq:agnostic.nonrelu.hatG.bound}
 \end{equation}
 Since $\pnorm{w-v}2\leq 1$ implies $\ell(w; x) = (1/2)(\sigma(w^\top x) - \sigma(v^\top x))^2 \leq L^2 B_X^2/2$, standard results from Rademacher complexity  (e.g., Theorem 26.5 of~\cite{shalevschwartz}) imply that with probability at least $1-\delta/2$, 
  \[ G(w_t) \leq \hat G(w_t) + \E_{S\sim \calD^n} \mathfrak{R}_S(\ell \circ \sigma \circ \calG) + 2 L^2 B_X^2 \sqrt{\f{ 2 \log(8/\delta)}{n}},\]
 where $\ell$ is the loss and $\calG$ is the function class defined in Lemma \ref{lemma:rademacher.complexity}.  
 We can combine \eqref{eq:agnostic.nonrelu.hatG.bound} with Lemma \ref{lemma:rademacher.complexity} and a union bound to get that with probability at least $1-\delta$,
  \[ G(w_t) \leq 6 L^3 \gamma^{-3} \l( \opt +  3a\sqrt{\f{ \log(4/\delta)}{n}}\r)+ \f{ 2 L^3 B_X^2}{\sqrt n} + \f{ 2L^2 B_X^2 \sqrt{2\log(8/\delta)}}{\sqrt n} .\]
This shows that $G(w_t) \leq O(\opt + n^{-1/2})$.  By Claim \ref{claim:Gbound:implies:Fbound}, we have
 \[ F(w_t) \leq 2 G(w_t) + 2 \OPT \leq O(\opt + n^{-1/2}),\]
completing the proof for those $\sigma$ satisfying Assumption \ref{assumption:activation.fcn}.\end{proof}

\subsection{ReLU activation}\label{sec:relu.activation}
The proof above crucially relies upon the fact that $\sigma$ is strictly increasing so that we may apply Fact \ref{fact:sigma.strictly.increasing} in the proof of Lemma \ref{lemma:agnostic.key.strictly.increasing}.  In particular, it is difficult to show a strong lower bound for the gradient direction term in \eqref{eq:ip.lb.strictly.increasing} if it is possible for $(z_1-z_2)^2$ to be arbitrarily large when $(\sigma(z_1)-\sigma(z_2))^2$ is small.  To get around this, we will use the same proof technique wherein we show that the gradient lower bound involves a term that relates the auxiliary error $\hat H(w_t)$ to $\hat F(v)$, but our bound will involve a term of the form $O(\hat F(v)^{1/2})$ rather than $O(\hat F(v))$.  To do so, we will use the following property of non-decreasing Lipschitz functions.

\begin{fact}\label{fact:sigma.L.lipschitz} 
If $\sigma$ is non-decreasing and $L$-Lipschitz, then for any $z_1, z_2$ in the domain of $\sigma$, it holds that $(\sigma(z_1) - \sigma(z_2))(z_1 - z_2) \geq L^{-1}(\sigma(z_1)-\sigma(z_2))^2$.  \end{fact}

With this fact we can present the analogue to Lemma \ref{lemma:agnostic.key.strictly.increasing} that holds for a general non-decreasing and Lipschitz activation and hence includes the ReLU. 
\begin{lemma}\label{lemma:agnostic.key.relu}
Suppose that $\pnorm{x}2 \leq B_X$ a.s. under $\calD_x$.  Suppose $\sigma$ is non-decreasing and $L$-Lipschitz.   Assume $\hat F(v) \in (0,1)$.  Gradient descent with fixed step size $\eta \leq (1/4) L^{-2} B_X^{-2}$ initialized at $w_0=0$ finds weights $w_t$ satisfying $\hat H(w_t) \leq 2 L^2 B_X \hat F(v)^{1/2}$ within $T = \ceil{\eta^{-1} L^{-1} B_X^{-1} \hat F(v)^{-1/2}}$  iterations, with $\pnorm {w_t-v}2 \leq 1$ for each $t=0, \dots, T-1$.
\end{lemma}
\begin{proof}
Just as in the proof of Lemma \ref{lemma:agnostic.key.strictly.increasing}, the lemma is proved if we can show that for every $t\in \N$, either (a) $\hat H(w_\tau) \leq 2 L^2 B_X \hat F(v)^{1/2}$ for some $\tau < t$, or (b) $\pnorm{w_t-v}2^2 \leq \pnorm{w_{t-1}-v}2^2 - \eta LB_X \hat F(v)^{1/2}$ holds.  To this end we assume the induction hypothesis holds for some $t\in \N$, and since we are done if (a) holds, we assume (a) does not hold and thus for every $\tau \leq t$, we have $\hat H(w_\tau) > 2 L^2 B_X\hat F(v)^{1/2}$.  Since (a) does not hold, $\pnorm{w_\tau -v}2^2 \leq \pnorm{w_{\tau-1}-v}2^2 - \eta L B_X \hat F(v)^{1/2}$ holds for each $\tau=1, \dots, t$ and hence the identity
\begin{equation}
    \pnorm{w_\tau - v}2 \leq 1 \quad \forall \tau \leq t,
    \label{eq:bounded.weights.agnostic.relu}
\end{equation}
holds.  We now proceed with showing the analogues of \eqref{eq:ip.lb.strictly.increasing} and \eqref{eq:grad.ub.strictly.increasing}.  We begin with the lower bound,
     \begin{align}
     \nonumber
         \big \langle \nabla \hat F(w_t) , w_t-v\big \rangle &= 
          (1/n) \summ i n \l( \sigma(w_t^\top x_i) - \sigma(v^\top x_i) \r) \sigma'(w_t^\top x_i) (w_t^\top x_i - v^\top x_i) \\
          \label{eq:agnostic.glm.key.difference}
         &\quad + \big \langle (1/n)  \summ i n \l( \sigma(v^\top x_i) - y_i \r) \sigma'(w_t^\top x_i) x_i , w_t-v\big \rangle\\
         \nonumber
         &\geq (1/Ln) \summ i n \l( \sigma(w_t^\top x_i) - \sigma(v^\top x_i) \r)^2 \sigma'(w_t^\top x_i)\\
         \nonumber
         &\quad - \pnorm{w_t-v}2 \bigg \| (1/n) \summ i n \l( \sigma(v^\top x_i) - y_i \r) \sigma'(w_t^\top x_i) x_i \bigg\|_2 \\
         &\geq 2 L^{-1} \hat H(w_t) - L B_X \hat F(v)^{1/2}.
         \label{eq:agnostic.glm.ip.lb}
     \end{align}
 In the first inequality, we have used Fact \ref{fact:sigma.L.lipschitz} and that $\sigma'(z) \geq 0$ for the first term.  For the second term, we use Cauchy--Schwarz.  The last inequality is a consequence of \eqref{eq:bounded.weights.agnostic.relu}, Cauchy--Schwarz, and that $\sigma'(z) \leq L$ and $\pnorm{x}2\leq B_X$.  
As for the gradient upper bound at $w_t$, the bound \eqref{eq:grad.ub.strictly.increasing} still holds since it only uses that $\sigma$ is $L$-Lipschitz.  The choice of $\eta\leq (1/4) L^{-2} B_X^{-2}$ then ensures
\begin{align}
\nonumber
    \pnorm{w_t-v}2^2 - \pnorm{w_{t+1}-v}2^2 &\geq 2 \eta \l( 2 L^{-1} \hat H(w_t)  - L B_X\hat F(v)^{1/2} \r) \\
    \nonumber
    &\quad - \eta^2 \l( 4B_X^2 L \hat H(w_t) + 4L^2 B_X^2 \hat F(v)\r) \\
    \nonumber
    &\geq \eta \l( 3L^{-1} \hat H(w_t) - 3 L B_X \l( \hat F(v) \vee \hat F(v)^{1/2} \r)\r)\\
    &\geq \eta L B_X \hat F(v)^{1/2},
\end{align}
where the last line comes from the induction hypothesis that $\hat H(w_t) \geq 2 L^2 B_X \hat F(v)^{1/2}$ and since $\hat F(v)\in (0,1)$.  This completes the proof. 
\end{proof}

With this lemma in hand, the proof of Theorem \ref{theorem:agnostic} follows just as in the strictly increasing case. 
\begin{proof}[Proof of Theorem \ref{theorem:agnostic} for ReLU]
 We highlight here the main technical differences with the proof for the strictly increasing case.  Although Lemma \ref{lemma:rademacher.complexity} applies to the loss function $\ell(w; x) = (1/2) \l(\sigma(w^\top x) - \sigma(v^\top x)\r)^2$, the same results hold for the loss function $\tilde \ell(w; x) = \ell(w; x) \sigma'(w^\top x)$ for ReLU, since $\nabla \sigma'(w^\top x) \equiv 0$ a.e. Thus $\tilde \ell$ is still $B_X$-Lipschitz, and we have
 \begin{equation}
     \E_{S\sim \calD^n} \mathfrak{R}_S \l( \tilde \ell \circ \sigma \circ \calG \r) \leq \f{ 2 B_X^2}{\sqrt n}.
     \label{eq:rademacher.relu}
 \end{equation}
 With this in hand, the proof is essentially identical: 
By Lemmas \ref{lemma:agnostic.key.relu} and \ref{lemma:F(v).opt.concentration}, with probability at least $1-\delta/2$ gradient descent finds a point with
 \begin{equation}
     \hat H(w_t) \leq 2 B_X \hat F(v)^{1/2} \leq 2  B_X \l( \opt^{1/2} + \f{ \sqrt {3a} \log^{1/4} (4/\delta)}{n^{1/4}}\r).
 \end{equation}
 We can then use \eqref{eq:rademacher.relu} to get that with probability at least $1-\delta$,
 \begin{equation}
    H(w_t) \leq 2  B_X \l( \opt^{1/2} + \f{ \sqrt {3a} \log^{1/4} (4/\delta)}{n^{1/4}}\r) + \f{ 2B_X^2}{\sqrt n} + 2B_X^2 \sqrt{\f{ 2\log(8/\delta)}n}.
\end{equation}
 Since $\calD_x$ satisfies Assumption \ref{assumption:marginal.spread} and $\pnorm{w_t-v}2\leq 1$, Lemma \ref{lemma:Hsurrogate} yields $G(w_t)\leq 8 \sqrt 2 \alpha^{-4} \beta^{-1} H(w_t)$.  Then applying Claim \ref{claim:Gbound:implies:Fbound} completes the proof.
\end{proof}
\begin{remark}
An examination of the proof of Theorem \ref{theorem:agnostic} shows that when $\sigma$ satisfies Assumption \ref{assumption:activation.fcn}, any initialization with $\pnorm{w_0-v}2$ bounded by a universal constant will suffice.  In particular, if we use Gaussian initialization $w_0\sim N(0,\tau^2 I_d)$ for $\tau^2=O(1/d)$, then by concentration of the chi-square distribution the theorem holds with (exponentially) high probability over the random initialization.  For ReLU, initialization at the origin greatly simplifies the proof since Lemma \ref{lemma:agnostic.key.relu} shows that $\pnorm{w_{t}-v}2\leq \pnorm{w_0-v}2$ for all $t$.  When $w_0=0$, this implies $\pnorm{w_t-v}2\leq 1$ and allows for an easy application of Lemma \ref{lemma:Hsurrogate}.  For isotropic Gaussian initialization, one can show that with probability approaching 1/2 that $\pnorm{w_0-v}2<1$ provided its variance satisfies $\tau^2 = O(1/d)$ (see e.g. Lemma 5.1 of~\citet{yehudai20}).  In this case, the theorem will hold with constant probability over the random initialization.
\end{remark}

\section{Noisy teacher network setting}
\label{sec:noisy}
In this section, we consider the teacher network setting, where the joint distribution of $(x,y)\sim \calD$ is given by a target neuron $v$ (with $\pnorm{v}2\leq 1$) plus zero-mean $s$-sub-Gaussian noise,
\[ y | x \sim \sigma(v^\top x) + \xi,\quad \E\xi|x =0.\]
 We assume throughout this section that $\xi\not \equiv 0$; we deal with the realizable setting separately (and achieve improved sample complexity) in Appendix \ref{appendix:realizable}.  We note that this is precisely the setup of the generalized linear model with (inverse) link function $\sigma$.  We further note that we only assume that $\E[y|x] = \sigma(v^\top x)$, i.e., the noise is \textit{not} assumed to be independent of the input $x$, and thus falls into the probabilistic concept learning model of~\citet{kearns.probabilistic}. 

With the additional structural assumption of a noisy teacher, we can improve the agnostic result from $O(\opt)+\eps$ (for strictly increasing activations) and $O(\opt^{1/2})+\eps$ (for ReLU) to $\opt+\eps$.  The key difference from the proof in the agnostic setting is that when trying to show the gradient points in a good direction as in \eqref{eq:ip.lb.strictly.increasing} and \eqref{eq:agnostic.glm.key.difference}, since we know $\E[y|x] = \sigma(v^\top x)$, the average of terms of the form $a_i (\sigma(v^\top x_i) - y_i)$ with fixed and bounded coefficients $a_i$ will concentrate around zero. This allows us to improve the lower bound from $\langle\nabla \hat F(w_t), w_t-v\rangle \geq C( \hat H(w) - \hat F(v)^{\alpha})$ to one of the form $\geq C( \hat H(w) - \eps)$, where $C$ is an absolute constant. 
The full proof of Theorem \ref{theorem:glm} is given in Appendix \ref{appendix:glm}.
\begin{theorem}\label{theorem:glm}
 Suppose $\calD_x$ satisfies $\pnorm{x}2\leq B_X$ a.s. and $\E[y | x] = \sigma(v^\top x)$ for some $\pnorm{v}2 \leq 1$.  Assume that $\sigma(v^\top x) - y$ is $s$-sub-Gaussian.   Assume gradient descent is initialized at $w_0=0$ and fix a step size $\eta \leq (1/4) L^{-2} B_X^{-2}$.   If $\sigma$ satisfies Assumption \ref{assumption:activation.fcn}, let $\gamma$ be the constant corresponding to $\rho = 2B_X$.  There exists an absolute constant $c_0>0$ such that for any $\delta >0$, with probability at least $1-\delta$, gradient descent for $T = \eta^{-1} \sqrt n / (c_0 LB_x s\sqrt{\log(4d/\delta)})$ iterations finds weights $w_t$, $t<T$, satisfying
 \begin{equation}
      F(w_t) \leq  \opt + C_1 n^{-1/2} + C_2 n^{-1/2} \sqrt{\log(8/\delta) }+  C_3 n^{-1/2}\sqrt{ \log(4d/\delta)},
      \label{eq:glm.thm.bound}
 \end{equation}
 where $C_1 = 4 L^3 B_X^2$, $C_2 = 2\sqrt 2 L^2 B_X^2 \sqrt 2$, and $C_3 = 4 c_0 \gamma^{-1} L^2 sB_X$.   
 When $\sigma$ is ReLU, further assume that $\calD_x$ satisfies Assumption \ref{assumption:marginal.spread} for constants $\alpha, \beta>0$, and let $\nu = \alpha^{4}\beta/8\sqrt 2$.  Then \eqref{eq:glm.thm.bound} holds for $C_1 = B_X^2\nu^{-1}$, $C_2 = 2C_1$, and $C_3 = 4 c_0 s \nu^{-1} B_X$. 
 \end{theorem}
 We first note that although \eqref{eq:glm.thm.bound} contains a $\log(d)$ term, the dependence on the dimension can be removed if we assume that 
 the noise is bounded rather than sub-Gaussian;  details for this are given in Appendix \ref{appendix:glm}.    As mentioned previously, if we are in the realizable setting, i.e. $\xi \equiv 0$, we can improve the sample and runtime complexities to $O(\eps^{-1})$ by using online SGD and a martingale Bernstein bound.  For details on the realizable case, see Appendix \ref{appendix:realizable}. 
 
  In comparison with existing literature,~\citet{kakade2011} proposed GLMTron to show the learnability of the noisy teacher network for a non-decreasing and Lipschitz activation $\sigma$ when the noise is bounded.\footnote{A close inspection of the proof shows that sub-Gaussian noise can be handled with the same concentration of norm sub-Gaussian random vectors that we use for our results.} In GLMTron, updates take the form $w_{t+1} = w_t - \eta \tilde g_t$ where $\tilde g_t = ( \sigma(w_t^\top x) - y)x$, while in gradient descent, the updates take the form $w_{t+1} = w_t - \eta g_t$ where $g_t = \tilde g_t \sigma'(w_t^\top x)$.  Intuitively, when the weights are in a bounded region and $\sigma$ is strictly increasing and Lipschitz, the derivative satisfies $\sigma'(w_t^\top x) \in [\gamma, L]$ and so the additional $\sigma'$ factor will not significantly affect the algorithm.   For ReLU this is more complicated as the gradient could in the worst case be zero in a large region of the input space, preventing effective learnability using gradient-based optimization, as was demonstrated in the negative result of~\citet{yehudai20}.  For this reason, a type of nondegeneracy condition like Assumption \ref{assumption:marginal.spread} is essential for gradient descent on ReLUs.  
  
  In terms of other results for ReLU, recent work by~\citet{mukherjee} introduced another modified version of SGD, where updates now take the form $w_{t+1}=w_t - \eta \hat g_t$, with $\hat g_t = \tilde g_t \sigma'(y>\theta)$, and $\theta$ is an upper bound for an adversarial noise term.  They showed that this modified SGD recovers the parameter $v$ of the teacher network under the nondegeneracy condition that the matrix $\E_x [ xx^\top \ind(v^\top x\geq 0)]$ is positive definite.  A similar assumption was used by~\citet{du2017} in the realizable setting. 
  
  Our GLM result is also comparable to recent work by~\citet{foster2018}, where the authors provide a meta-algorithm for translating guarantees for $\eps$-stationary points of the empirical risk to guarantees for the population risk provided that the population risk satisfies the so-called ``gradient domination'' condition and the algorithm can guarantee that the weights remain bounded (see their Proposition 3).  By considering GLMs with bounded, strictly increasing, Lipschitz activations, they show the gradient domination condition holds, and any algorithm that can find a stationary point of an $\ell^2$-regularized empirical risk objective is guaranteed to have a population risk bound.  In contrast, our result concretely shows that vanilla gradient descent learns the GLM, even in the ReLU setting.

\section{Conclusion and remaining open problems}
In this work, we considered the problem of learning a single neuron with the squared loss by using gradient descent on the empirical risk.  We first analyzed this in the agnostic PAC learning framework and showed that if the activation function is strictly increasing and Lipschitz, then gradient descent finds weights with population risk $O(\opt) + \eps$, where $\opt$ is the smallest possible population risk achieved by a single neuron. When the activation function is ReLU, we showed that gradient descent finds a point with population risk at most $O(\opt^{1/2})+\eps$.  Under the more restricted noisy teacher network setting, we showed the population risk guarantees improve to $\opt+\eps$ for both strictly increasing activations and ReLU.

Our work points towards a number of open problems.  Does gradient descent on the empirical risk provably achieve population risk with a better dependence on $\opt$ than we have shown in this work, or are there distributions for which this is impossible?  Recent work by~\citet{goel2020sqlowerbounds} provides a statistical query lower bound for learning a sigmoid with respect to the correlation loss $\E[\ell(y \sigma(w^\top x))]$, but we are not aware of lower bounds for learning non-ReLU single neurons under the squared loss.  It thus remains a possibility that gradient descent (or another algorithm) can achieve $\OPT+\eps$ risk for such activation functions.  For ReLU,~\cite{diakonikolas2020approximation} showed that gradient descent on a convex surrogate for the empirical risk can achieve $O(\opt) + \eps$ population risk for log concave distributions; it would be interesting if such bounds could be shown for gradient descent on the empirical risk itself.


\section*{Acknowledgement}
We thank Adam Klivans for his helpful comments on our work. We also thank Surbhi Goel for pointing out how to extend the results of~\citet{diakonikolas2020approximation} to more general distributions and to leaky-ReLU-type activation functions.  

\appendix

\section{Detailed comparisons with related work}
\label{appendix:comparisons}

Here, we describe comparisons of our results to those in the literature and give detailed comments on the specific rates we achieve.  In Table \ref{table:agnostic}, we compare our agnostic learning results.  We note the guarantees for the population risk in the fourth column, the marginal distributions over $x$ for which the bounds hold in the fifth column, and the sample complexity required to reach the specified level of risk plus some $\eps>0$ in the final column.   Our results in this setting come from Theorem \ref{theorem:agnostic}.  The Big-O notation hides constants that may depend on the parameters of the distribution or activation function, but does not hide explicit dependence on the dimension $d$.  However, the parameters of the distribution itself may have \textit{implicit} dependence on the dimension.  In particular, for bounded distributions that satisfy $\pnorm{x}2\leq B_X$, the $O()$ hides multiplicative factors that depend on $B_X$.  This means that if $B_X$ depends on $d$, so will our bounds.   For ReLU, the $O()$ hides polynomial factors in $B_X$.  
For non-ReLU, the worst-case activation functions under consideration in Assumption \ref{assumption:activation.fcn} (e.g. the sigmoid) can have $\gamma\sim \exp(-B_X)$, making the runtime and sample complexity depend on $\gamma^{-1} \sim \exp(B_X)$, in which case it is preferable for $B_X$ to be a constant independent of the dimension.  We note that the sample complexity for~\citet{diakonikolas2020approximation} for the $(1+\delta)\opt$ guarantee is $O(\eps^{-2} [d\delta^{-3} \nu^{-2}]^{\delta^{-3}})$ when $\calD_x$ is $\nu$ sub-Gaussian for some $\nu = O(1)$, and thus the exact dependence on the dimension depends on the sub-Gaussian norm and error threshold desired. 

In Table \ref{table:glm}, we provide comparisons of our noisy teacher network setting, where $y = \sigma(v^\top x) + \xi$ for some zero mean noise $\xi$.  Our results in this setting come from Theorem \ref{theorem:glm}.  The complexity column here denotes the sample complexity required to reach population risk $\opt+\eps$.  The subspace eigenvalue assumption given by~\citet{mukherjee} is that $\E[xx^\top \ind(v^\top x \geq 0)]\mgtr 0$.  We note that the result of Mukherjee and Muthukumar holds for any bounded noise distribution and thus is in the more general adversarial noise (but not agnostic\footnote{Agnostic learning results typically require i.i.d. samples, and adversarial noise may depend on other samples in malicious ways.  Even in the i.i.d. case, trouble arises if one wishes to use parameter recovery to show that a given algorithm competes with the population risk minimizer.  Consider the ReLU with labels given by $y = \sigma(v^\top x) + \xi$ where $\xi = -\sigma(v^\top x)$.  The zero vector minimizes the population risk, and so any algorithm that returns the target neuron $\sigma(v^\top x)$ has large population risk.  A similar phenomenon occurs for $\xi = \sigma(v^\top x)$.}) setting.

Finally, in Table \ref{table:realizable}, we provide comparisons with results in the realizable setting ($\xi \equiv 0)$.  (Our results in this setting are given in Theorem \ref{theorem:gd.loss} in Appendix \ref{appendix:realizable}.)  For G.D. and S.G.D., the complexity column denotes the sample complexity required to reach population risk $\eps$.  For G.D. or gradient flow on the population risk, it refers to the runtime complexity only as there are no samples in this setting.   For~\citet{du2017}, the subspace eigenvalue assumption is that for any $w$ and for the target neuron $v$, it holds that $\E[xx^\top \ind(w^\top x \geq 0, v^\top x \geq)]\mgtr 0$.  This is a nondegeneracy assumption that is related to the marginal spread condition given in Assumption \ref{assumption:marginal.spread}, in the sense that it allows for one to show that $H$ is an upper bound for $G$.  Finally, we note that any result in the agnostic or noisy teacher network settings applies in the realizable setting as well. 

\begin{savenotes}
\begin{table}[!ht]
\centering
\caption{Comparison of results in the agnostic setting}
\begin{tabular}[t]{p{3.7cm} p{2.4cm} p{2.2cm} p{2.5cm} p{2.4cm}}
\addlinespace
\toprule
\multirow{2}{*}{Algorithm} & 
\multirow{2}{*}{Activations} &
\multirow{2}{*}{Pop. risk} &
\multirow{2}{*}{$\calD_x$} &
\multirow{2}{*}[5pt]{Sample} \\
&&&&Complexity \\
\midrule
Halfspace reduction\newline \citep{goel2019relugaussian} & ReLU & $O(\opt^{2/3})$ &
standard\newline Gaussian & $O(\mathrm{poly}(d,\eps^{-1}))$ \\
\addlinespace
Convex surrogate G.D.\newline
\citep{diakonikolas2020approximation}\footnote{Although their result is stated for the ReLU and isotropic log-concave distributions, their results also apply for $L$-Lipschitz activations satisfying $\inf_z \sigma'(z) \geq \gamma>0$ for isotropic distributions that satisfy our Assumption \ref{assumption:marginal.spread}.  In this setting, one can show that the Chow parameters satisfy $\norm{\chi(\sigma_u) - \chi(\sigma_w)}^2 \geq \gamma L^{-1} \E[(\sigma(u^\top x) - \sigma(v^\top x))^2]$, from which the result follows easily.} & ReLU & $O(\OPT)$ & isotropic\newline +log-concave & $O(d\eps^{-2})$ \\
\addlinespace
Convex surrogate G.D.\newline
+ Domain Partition\newline
\citep{diakonikolas2020approximation} & ReLU & $(1+\delta)\OPT$ & sub-Gaussian & $O(d^c \eps^{-2})$ \\
\addlinespace
Gradient Descent \newline (This paper)&
strictly\newline increasing\newline + Lipschitz
& $O(\opt)$
& bounded 
 & $O(\eps^{-2})$ \\
 \addlinespace
Gradient Descent\newline (This paper) & ReLU & $O(\opt^{1/2})$ &
bounded \newline + marginal\newline spread & $O(\eps^{-4})$\\
\bottomrule
\addlinespace
\end{tabular}
\label{table:agnostic}
\end{table}
\end{savenotes}

\begin{table}[!ht]
\centering
\caption{Comparison of results in the noisy teacher network setting }
\begin{tabular}{p{6.1cm} p{3.1cm} p{3.5cm} p{2.3cm}}
\addlinespace
\toprule
\multirow{2}{*}{Algorithm} & 
\multirow{2}{*}{Activations} &
\multirow{2}{*}{$\calD_x$} &
\multirow{2}{*}[5pt]{Sample} \\
&&&Complexity \\
\midrule
GLMTron\newline \citep{kakade2011} & increasing\newline + Lipschitz & bounded & $O(\eps^{-2})$ \\
\addlinespace
Modified Stochastic Gradient Descent\newline \citep{mukherjee} & ReLU & bounded \newline + subspace eigenvalue & $O(\log(1/\eps))$ \\
\addlinespace
Meta-algorithm\newline \citep{foster2018} & strictly\newline increasing \newline + Lipschitz \newline + $\sigma'$ Lipschitz & bounded & $O(\eps^{-2} \wedge d\eps^{-1})$ \\
\addlinespace
Gradient Descent\newline \citep{mei2016landscape}& strictly increasing \newline + diff'ble \newline + Lipschitz\newline + $\sigma'$ Lipschitz\newline + $\sigma''$ Lipschitz & centered \newline + sub-Gaussian \newline + $\E[xx^\top]\mgtr 0$ & $O(d\eps^{-1})$\\
\addlinespace
\addlinespace
Gradient Descent\newline (This paper) & strictly increasing \newline + Lipschitz  & bounded & $O(\eps^{-2})$ \\
\addlinespace
Gradient Descent\newline (This paper) &  ReLU & bounded \newline + marginal spread  & $O(\eps^{-2})$\\
\bottomrule
\addlinespace
\end{tabular}
\label{table:glm}
\end{table}

\begin{table}[!ht]
\centering
\caption{Comparison of results in the realizable setting}
\begin{tabular}{p{5.0cm} p{3.6cm} p{3.6cm} p{2.4cm}}
\addlinespace
\toprule
\multirow{2}{*}{Algorithm} & 
\multirow{2}{*}{Activations} &
\multirow{2}{*}{$\calD_x$} &
\multirow{2}{*}[5pt]{Sample} \\
&&&Complexity \\
\midrule
Stochastic Gradient Descent\newline \citep{du2017} & ReLU & bounded\newline + subspace eigenvalue & $O(\log(1/\eps))$ \\
\addlinespace
Projected Regularized\newline Gradient Descent \newline \citep{soltanolkotabi2017relus} & ReLU & standard\newline Gaussian & $O(\log(1/\eps))$ \\
\addlinespace
Population Gradient Descent\newline \citep{yehudai20} & $\inf_{z\in \R} \sigma'(z) > 0$ &  bounded\newline + $\E[xx^\top]\mgtr 0$ & $O(\log(1/\eps))$\\
\addlinespace
Population Gradient Descent\newline \citep{yehudai20} & $\inf_{0 < z < \alpha} \sigma'(z)>0$ \newline+ Lipschitz  & bounded\newline+ marginal spread & $O(\log(1/\eps))$ \\
\addlinespace
Population Gradient Flow\newline \citep{yehudai20} & ReLU & marginal spread\newline + spherical symmetry & $O(\log(1/\eps))$ \\
\addlinespace
Stochastic Gradient Descent\newline \citep{yehudai20} & $\inf_{0 < z < \alpha} \sigma'(z)>0$ \newline+ Lipschitz  & bounded\newline+ marginal spread & $\tilde O(\eps^{-2})$ \\
\addlinespace
\addlinespace
Population Gradient Descent\newline+ Stochastic Gradient Descent\newline (This paper) & strictly increasing\newline + Lipschitz  & bounded & $O(\eps^{-1})$ \\
\addlinespace
Population Gradient Descent\newline+ Stochastic Gradient Descent\newline (This paper)  &  ReLU & bounded\newline + marginal spread & $O(\eps^{-1})$\\
\bottomrule
\addlinespace
\end{tabular}
\label{table:realizable}
\end{table}

\section{Proof of Lemma \ref{lemma:Hsurrogate}}
\label{appendix:Hsurrogate}
To prove Lemma \ref{lemma:Hsurrogate}, we use the following result of~\citet{yehudai20}. 
\begin{lemma}[Lemma B.1, ~\citep{yehudai20}]\label{lemma:yehudai.lemmab1}
Under Assumption \ref{assumption:marginal.spread}, for any two vectors $a,b\in \R^2$ satisfying $\theta(a,b) \leq \pi-\delta$ for $\delta \in (0,\pi]$, it holds that
\[ \inf_{u\in \R^2 :\ \norm u=1} \int (u^\top y)^2 \ind(a^\top y \geq 0,\ b^\top y \geq 0,\ \norm y \leq \alpha) dy \geq \f{\alpha^4}{8 \sqrt 2} \sin^3(\delta/4).\]
\end{lemma}

\begin{proof}[Proof of Lemma \ref{lemma:Hsurrogate}]

 We first consider the case when $\sigma$ satisfies Assumption \ref{assumption:activation.fcn}.  By assumption,
\[ \oF(w) = (1/2)\E \l[ \l(\sigma(w_t^\top x) - \sigma(v^\top x)\r)^2 \sigma'(w_t^\top x)  \r] \leq \eps.\]
 Since the term in the expectation is nonnegative, restricting the integral to a smaller set only decreases its value, so that
\begin{equation}
(1/2)\E \l[ \l(\sigma(w_t^\top x) - \sigma(v^\top x)\r)^2 \sigma'(w_t^\top x) \ind(|w_t^\top x| \leq \rho) \r] \leq \eps.
\label{eq:main.ofsmallnorm.bound}
\end{equation}
For $\rho = BW$, since $\pnorm{w}2\leq W$, the inclusion $\{ \pnorm x 2 \leq \rho/W\} = \{ \pnorm x 2 \leq B \} \subset \{ |w_t^\top x| \leq \rho \}$
holds.  This means we can lower bound \eqref{eq:main.ofsmallnorm.bound} by substituting the indicator $\ind(|w_t^\top x| \leq \rho)$ with $\ind(\pnorm x2 \leq B)$, which is identically one by assumption.  Since $H(w)\leq \eps$, this implies
\[ 
\f \gamma 2\E\l[ \l( \sigma(w_t^\top x) - \sigma(v^\top x)\r)^2\r] \leq (1/2)\E \l[ \l(\sigma(w_t^\top x) - \sigma(v^\top x)\r)^2 \sigma'(w_t^\top x) \ind(\pnorm x 2 \leq B) \r] \leq \eps.
\]
Dividing both sides by $\gamma$ completes this part of the proof.

For ReLU, let us assume that $\oF(w) \leq \eps$, and denote the event
\[ K_{w,v} := \{w^\top x \geq 0, v^\top x \geq 0 \},\]
and define $\zeta := \beta \alpha^4 /8 \sqrt 2$.  Since $\oF(w) = \E[(\sigma(w^\top x) - \sigma(v^\top x))^2 \ind(w^\top x \geq 0)] \leq \zeta \eps$, it holds that
\begin{equation}
    \E \l[ \l( \sigma(w^\top x) - \sigma(v^\top x) \r)^2 \ind(K_{w,v})\r] \leq \zeta \eps.
    \label{eq:sq.loss.subspace.bound}
\end{equation}
Denote $\hat w$ and $\hat v$ as the projections of $w$ and $v$ respectively onto the two dimensional subspace $\mathrm{span}(w,v)$.  Using a proof similar to that of~\citet{yehudai20}, we have
\begin{align}
\nonumber
    &\E_{x\sim \calD}\l[ \l(w^\top x - v^\top x\r)^2 \ind(K_{w,v})\r]  = \pnorm{w-v}2^2 \E_{x\sim \calD} \l[ \l( \l( \f{w-v}{\pnorm{w-v}2}\r)^\top x\r)^2 \ind(K_{w,v})\r]\\
    \nonumber
    &\geq \pnorm{w-v}2^2 \inf_{u\in \mathrm{span}(w,v),\ \norm u =1} \E_x \l[ \ind(u^\top x)^2 \ind(K_{w,v})\r]\\
    \nonumber
    &= \pnorm{w-v}2^2 \inf_{u\in \R^2, \ \norm{u} =1} \E_{y\sim \calD_{w,v}} \l[ (u^\top y)^2 \ind(\hat w^\top y \geq 0, \ \hat v^\top y \geq 0)\r]\\
    \nonumber
    &\geq \pnorm{w-v}2^2 \inf_{u\in \R^2,\ \norm u=1} \int (u^\top y)^2 \ind(\hat w^\top y\geq 0,\ \hat v^\top y \geq 0,\ \pnorm y 2 \leq \alpha) p_{w,v}(y) dy \\
    &\geq \beta \pnorm{w-v}2^2 \inf_{u\in \R^2,\ \norm u=1} \int (u^\top y)^2 \ind(\hat w^\top y\geq 0,\ \hat v^\top y \geq 0,\ \pnorm y 2 \leq \alpha) dy.
    \label{eq:relu.lowerbound.intermediate}
\end{align}
By assumption, $\pnorm{w-v}2 \leq 1$.  Since
\[ 1 \geq \pnorm{w-v}2^2 = \pnorm{w}2 \l( \pnorm{w}2 - 2 \cos \theta(w,v) \r) + 1,\]
we must have either $w=0$ or $\theta(w,v) \in [0,\pi/2]$.   To see that $w=0$ is impossible, suppose for the contradiction that $w=0$ and so $\oF(w) = \oF(0)\leq \zeta \eps$.  Let $z$ be any vector orthogonal to $v$, so that $\theta(v,z) = \pi/2$.  Then,
\begin{align}
\nonumber
    \zeta \eps &\geq \oF(0) \\
    \nonumber
    &= \E_{x\sim \calD}\l[ (v^\top x)^2 \ind(v^\top x \geq 0) \r] \\
    \nonumber
    &= \E_{y\sim \calD_{0,v}} \l[ (\hat v^\top y)^2 \ind(\hat v^\top y \geq 0 \r] \\
    \nonumber
    &\geq \inf_{u:\ \norm u=1} \int (u^\top x)^2 \ind(v^\top x \geq 0, z^\top x \geq 0, \pnorm y2\leq \alpha) p_{0,v}(y) dy \\
    \nonumber
    &\geq \beta \inf_{u:\ \norm u=1} \int (u^\top x)^2 \ind(v^\top x \geq 0, z^\top x \geq 0, \pnorm y2\leq \alpha) dy \\
    &\geq \f{\beta \alpha^4}{8 \sqrt 2}.
\end{align}
The last line follows by using Lemma \ref{lemma:yehudai.lemmab1}.  For $\eps<1$, this is impossible by the definition of $\zeta$.  This shows that $\theta(w,v) \leq \pi/2$.  We can therefore apply Lemma \ref{lemma:yehudai.lemmab1} to \eqref{eq:relu.lowerbound.intermediate} to get
\begin{align*}
    \zeta \eps &\geq \beta \pnorm{w-v}2^2 \inf_{u\in \R^2,\ \norm u=1} \int (u^\top y)^2 \ind(\hat {w}^\top y\geq 0,\ \hat v^\top y \geq 0,\ \pnorm y 2 \leq \alpha) dy\\
    &\geq \f{\beta \alpha^4}{8 \sqrt 2} \pnorm{w -v}2^2 \\
    &= \zeta B^2 \pnorm{w-v}2^2.
\end{align*}
This shows that $\pnorm{w-v}2^2 \leq B^{-2} \eps$.  Since $\sigma$ is 1-Lipschitz, H\"older's inequality and $\E\pnorm{x}2^2 \leq B^2$ imply that $G(w) \leq \eps$.
\end{proof}

\section{Noisy teacher network proofs}
\label{appendix:glm}
As in the agnostic case, we have a key lemma that shows $\hat H$ is small when we run gradient descent for a sufficiently large time.  Note that one difference with the proof in the agnostic case is that we do not need to consider different auxiliary errors for the strictly increasing and ReLU cases; $H$ alone suffices. 
 \begin{lemma}\label{lemma:glm.key}
 Suppose that $\pnorm{x}2\leq B_X$ a.s. under $\calD_x$.  Let $\sigma$ be non-decreasing and $L$-Lipschitz.  Suppose that the bound
 \begin{equation} 
 \| (1/n) \textstyle \summ i n \l(\sigma(v^\top x_i) - y_i \r) \alpha_i  x_i \|_2 \leq K \leq 1.
 \label{eq:Kbound.defn}
 \end{equation}
 holds for scalars satisfying $\alpha_i\in [0,L]$.  Then gradient descent run with fixed step size $\eta \leq (1/4) L^{-2} B_X^{-2}$ from initialization $w_0=0$ finds weights $w_t$ satisfying $\hat H(w_t) \leq 4 LK$ within $T = \ceil{\eta^{-1} K^{-1}}$ iterations, with $\pnorm{w_t-v}2 \leq 1$ for each $t=0,\dots, T-1$.
 \end{lemma}
 
\begin{proof}
 Just as in the proof of Lemma \ref{lemma:agnostic.key.strictly.increasing}, the theorem can be shown by proving the following induction statement.  We claim that for every $t\in \N$, either (a) $\hat H(w_\tau) \leq 4 LK$ for some $\tau < t$, or (b) $\pnorm{w_t-v}2^2 \leq \pnorm{w_{t-1}-v}2^2 - \eta K$.  If the induction hypothesis holds, then until gradient descent finds a point where $\hat H(w_t) \leq 4LK$, the squared distance $\pnorm{w_t-v}2^2$ decreases by $\eta K$ at every iteration.  Since $\pnorm{w_0-v}2^2 = 1$, this means there can be at most $\eta^{-1} K^{-1}$ iterations until we reach $\hat H(w_t) \leq 4 LK$.  This shows the induction statement implies the theorem. 
 
 We begin with the proof by supposing the induction hypothesis holds for $t$, and considering the case $t+1$.  If (a) holds, then we are done.  So now consider the case that for every $\tau \leq t$, we have $\hat H(w_\tau) > 4 LK$.   Since (a) does not hold, $\pnorm{w_\tau - v}2^2 \leq \pnorm{w_{\tau-1}-v}2^2 - \eta K$ holds for each $\tau = 1, \dots, t$.   Since $\pnorm{w_0-v}2=1$, this implies 
     \begin{equation}
         \pnorm{w_\tau -v}2 \leq 1 \ \forall \tau \leq t.
         \label{eq:glm.weights.bounded}
     \end{equation}

We can therefore bound
     \begin{align}
     \nonumber
         \ip{\nabla \hat F(w_t) } {w_t-v} &= \ip{\f 1 n \summ 1 n \l( \sigma(w_t^\top x_i) - y_i \r) \sigma'(w_t^\top x_i) x_i}{w_t-v} \\
         \nonumber
         &= \f 1 n \summ i n \l( \sigma(w_t^\top x_i) - \sigma(v^\top x_i) \r) \sigma'(w_t^\top x_i) (w_t^\top x_i - v^\top x_i) \\
         \nonumber
         &\quad + \ip{ \f 1 n \summ i n \l( \sigma(v^\top x_i) - y_i \r) \sigma'(w_t^\top x_i) x_i}{w_t-v}\\
         \nonumber
         &\geq \f{L^{-1}}n \summ i n \l( \sigma(w_t^\top x_i) - \sigma(v^\top x_i) \r)^2 \sigma'(w_t^\top x_i) - K \pnorm{w_t-v}2\\
         &\geq 2L^{-1} \hat H(w_t) - K.
         \label{eq:glm.ip.lb}
     \end{align}
 In the first inequality, we have used Fact \ref{fact:sigma.L.lipschitz} for the first term.  For the second term, we use \eqref{eq:Kbound.defn} and that $\alpha_i:= \sigma'(w_t^\top x_i) \in [0,L]$.  The last inequality uses \eqref{eq:glm.weights.bounded}.  
 
 For the gradient upper bound, we have
 \begin{align}
 \nonumber
     \pnorm{\nabla \hat F(w_t)}2^2 &= \pnorm{\f 1 n \summ i n \l( \sigma(w_t^\top x_i) - \sigma(v^\top x_i) \r) \sigma'(w_t^\top x_i) x_i  + \f  1n \summ i n \l( \sigma(v^\top x_i) - y_i\r) \sigma'(w_t^\top x_i) x_i}2^2 \\
     \nonumber 
     &\leq 2\pnorm{\f 1 n \summ i n \l( \sigma(w_t^\top x_i) - \sigma(v^\top x_i) \r) \sigma'(w_t^\top x_i) x_i}2^2 \\
     \nonumber &\quad + 2\pnorm{ \f  1n \summ i n \l( \sigma(v^\top x_i) - y_i\r) \sigma'(w_t^\top x_i) x_i}2^2 \\
     \nonumber
     &\leq \f {2 L B_X^2}n \summ i n \l( \sigma(w^\top x_i) - \sigma(v^\top x_i) \r)^2 \sigma'(w_t^\top x_i) + 2K^2\\
     &= 4 L B_X^2 \hat H(w_t) + 2K^2.
     \label{eq:glm.grad.ub}
 \end{align}
 The first inequality uses Young's inequality.  The second uses that $\sigma'(z)\leq L$ and that $\pnorm{x}2\leq B_X$ a.s. and \eqref{eq:Kbound.defn}.  
 
 Putting \eqref{eq:glm.ip.lb} and \eqref{eq:glm.grad.ub} together, the choice of $\eta \leq (1/4) L^{-2} B_X^{-2}$ gives us
 \begin{align}
 \nonumber
     \pnorm{w_t-v}2^2 - \pnorm{w_{t+1}-v}2^2 &= 2 \eta \ip{\nabla \hat F(w_t)}{w_t-v} - \eta^2 \pnorm{\nabla \hat F(w_t)}2^2 \\
     \nonumber
     &\geq 2 \eta(L^{-1} \hat H(w_t) - K) - \eta^2 \l( 4L B_X^2 \hat H(w_t) + 2K^2 \r) \\
     \nonumber
     &\geq \eta L^{-1} \hat H(w_t) - 3 \eta K.
 \end{align}
 In particular, this implies
 \begin{equation}
      \pnorm{w_{t+1}-v}2^2 \leq \pnorm{w_t-v}2^2 + 3\eta K - \eta  L^{-1} \hat H(w_t)
      \label{eq:glm.wt+1.ub}
 \end{equation}
 Since $\hat H(w_t) > 4 K L$, this completes the induction.  The base case follows easily since $\pnorm{w_0-v}2 =1$ allows for us to deduce the desired bound on $\pnorm{w_1-v}2^2$ using \eqref{eq:glm.wt+1.ub}.  
 \end{proof}
 
 To prove a concrete bound on the $K$ term of Lemma \ref{lemma:glm.key}, we will need the following definition of norm sub-Gaussian random vectors.
 \begin{definition}
 A random vector $z\in \R^d$ is said to be \textit{norm sub-Gaussian with parameter $s >0$} if 
 \[ \P(\norm{z - \E z} \geq t) \leq 2 \exp (-t^2 / 2 s^2).\]
 \end{definition}
 A Hoeffding-type inequality for norm sub-Gaussian vectors was recently shown by~\citet{jin20.normsubgaussian}. 
 \begin{lemma}[Lemma 6,~\citep{jin20.normsubgaussian}]\label{lemma:normsubgaussian.hoeffding}
 Suppose $z_1, \dots, z_n\in \R^d$ are random vectors with filtration $\calF_t := \sigma(z_1, \dots, z_t)$ such that $z_i | \calF_{i-1}$ is a zero-mean norm sub-Gaussian vector with parameter $s_i\in \R$ for each $i$.  Then, there exists an absolute constant $c>0$ such that for any $\delta>0$, with probability at least $1-\delta$,
 \[ \norm{\summ i n z_i }\leq c \sqrt{\log(2d/\delta) \summ i n s_i^2} .\]
 \end{lemma}
 Using this, we can show that if $\xi_i := \sigma(v^\top x_i) - y_i$ is $s$ sub-Gaussian, then we can get a bound on $K$ at rate $n^{-1/2}$.  We note that if we make the stronger assumption that $\xi_i$ is bounded a.s., we can get rid of the $\log(d)$ dependence by using concentration of bounded random variables in a Hilbert space (e.g.~\citet{pinelis1986}, Corollary 2).  
 
 \begin{lemma}\label{lemma:Kbound.subgaussian.noise}
 Suppose that $\pnorm{x}2 \leq B_X$ a.s. under $\calD_x$, and let $\sigma$ be any continuous function.  Assume $\xi_i := \sigma(v^\top x_i) - y_i$ is $s$ sub-Gaussian and satisfies $\E[\xi_i| x_i]=0$.  Then there exists an absolute constant $c_0>0$ such that for constants $\alpha_i \in [0,L]$, with probability at least $1-\delta$, we have
 \[ \|(1/n) \textstyle \summ i n \l( \sigma(v^\top x_i) - y_i\r) \alpha_i x_i \| \leq c_0 L B_X s\sqrt{n^{-1} \log(2d/\delta)}.\]
 \end{lemma}
 
 \begin{proof}[Proof of Lemma \ref{lemma:Kbound.subgaussian.noise}]
 Define $z_i := \l(\sigma(v^\top x_i) - y_i \r) \alpha_i x_i$.  Using iterated expectations, we see that $\E[z_i]=0$.  Since $\sigma(v^\top x_i) - y_i$ is $s$-sub-Gaussian and $\pnorm{\alpha_i x_i}2 \leq L B_X$, it follows from the definition that $z_i$ is norm sub-Gaussian with parameter $LB_Xs$ for each $i$.  By Lemma \ref{lemma:normsubgaussian.hoeffding}, we have with probability at least $1-\delta$,
 \[  \norm{\summ i n z_i} \leq c \sqrt{ \log(2d/\delta) L^2 B_X^2 ns^2}.\]
 Dividing each side by $n$ proves the lemma.
 \end{proof}

 \begin{proof}[Proof of Theorem \ref{theorem:glm}]
 By Lemmas \ref{lemma:glm.key} and \ref{lemma:Kbound.subgaussian.noise}, there exists some $w_t$, $t < T$ and $\pnorm{w_t-v}2\leq 1$, such that
 \begin{equation}
 \nonumber
     \hat H(w_t) \leq 4LK \leq 4 c_0 L^2 B_X s \sqrt{\f{\log(2d/\delta)}n}.
 \end{equation}
 Consider $\sigma$ satisfying Assumption \ref{assumption:activation.fcn} first, with $\gamma$ corresponding to $\rho=2B_X$.  Since $\pnorm{w_t}2\leq 2$, we can use Lemma \ref{lemma:Hsurrogate} to transform the above bound for $\hat H$ into one for $\hat G$,
 \begin{equation}
 \nonumber
     \hat G(w_t) \leq 4 c_0 \gamma^{-1} L^2 B_Xs\sqrt{\f{\log(2d/\delta)}n}.
     \label{eq:nonrelu.hatG.bound}
 \end{equation}
 Since $\pnorm{w-v}2\leq 1$ implies $G(w) \leq L^2 B_X^2/2$, standard results from Rademacher complexity imply (e.g. Theorem 26.5 of~\cite{shalevschwartz}) that with probability at least $1-\delta$, 
 \[ G(w_t) \leq \hat G(w_t) + \E_{S\sim \calD^n} \mathfrak{R}_S(\ell \circ \sigma \circ \calG) + 2 L^2 B_X^2 \sqrt{\f{ 2 \log(4/\delta)}{n}},\]
 where $\ell(w; x) = (1/2) (\sigma(w^\top x) - \sigma(v^\top x))^2$ and $\calG$ are from Lemma \ref{lemma:rademacher.complexity}.  
 For the second term above, Lemma \ref{lemma:rademacher.complexity} and rescaling $\delta$ yields that 
 \[ G(w_t) \leq  \f{ 2 L^3 B_X^2}{\sqrt n} + \f{ 2L^2 B_X^2 \sqrt{2\log(8/\delta)}}{\sqrt n} + \f{ 4 c_0 \gamma^{-1} L^2 B_X s \sqrt{\log(4d/\delta)}}{\sqrt n}.\]
 Then Claim \ref{claim:Gbound:implies:Fbound} completes the proof for strictly increasing $\sigma$.
 
 When $\sigma$ is ReLU, the proof follows the same argument given in the proof of Theorem \ref{theorem:agnostic}.  Denoting the loss function $\tilde \ell(w; x) = (1/2)(\sigma(w^\top x) - \sigma(v^\top x))^2 \sigma'(w^\top x)$, we have
 \begin{equation}
     \E_{S\sim \calD^n} \mathfrak{R}_S \l( \tilde \ell \circ \sigma \circ \calG \r) \leq \f{ 2 B_X^2}{\sqrt n}.
     \label{eq:rademacher.relu.glm}
 \end{equation}
 By Lemmas \ref{lemma:glm.key} and \ref{lemma:Kbound.subgaussian.noise}, there exists some $w_t$, $t < T$ and $\pnorm{w_t-v}2\leq 1$, such that
 \begin{equation}
     \hat H(w_t) \leq 4LK \leq 4 c_0 L^2 B_X s \sqrt{\f{\log(2d/\delta)}n}.
 \end{equation}
 Using standard results from Rademacher complexity, 
  \[ H(w_t) \leq \hat H(w_t) + \E_{S\sim \calD^n} \mathfrak{R}_S(\tilde \ell \circ \sigma \circ \calG) + 2 B_X^2 \sqrt{\f{ 2 \log(4/\delta)}{n}}.\]
 By \eqref{eq:rademacher.relu.glm}, this means
 \[ H(w_t) \leq   \f{ 4 c_0  B_X s \sqrt{\log(4d/\delta)}}{\sqrt n} + \f{ 2 B_X^2}{\sqrt n} + \f{ 2 B_X^2 \sqrt{2\log(8/\delta)}}{\sqrt n}.\]
 
 Since $\calD$ satisfies Assumption \ref{assumption:marginal.spread} and $\pnorm{w_t-v}2\leq 1$, Lemma \ref{lemma:Hsurrogate} shows that $G(w_t) \leq 8\sqrt 2\alpha^{-4}\beta^{-1} H(w_t)$.   Then Claim \ref{claim:Gbound:implies:Fbound} translates the bound for $G(w_t)$ into one for $F(w_t)$. 
 \end{proof}

\section{Realizable setting}
\label{appendix:realizable}
In this section we assume $y = \sigma(v^\top x)$ a.s. for some $\pnorm v2 \leq 1$.  As in the agnostic and noisy teacher network setting, we use the auxiliary loss
\[ H(w) := (1/2) \E_{x\sim \calD} [(\sigma(w^\top x) - \sigma(v^\top x))^2 \sigma'(w^\top x)].\]
Note that in the realizable setting, the previous auxiliary loss $G$ defined in \eqref{def:auxiliary.loss} coincides with the true objective $F$, i.e. we have
\[ F(w) := (1/2) \E_{x\sim \calD} [(\sigma(w^\top x) - \sigma(v^\top x))^2].\]
For purpose of comparison with~\citet{yehudai20}, we provide analyses for two settings in the realizable case: in the first setting, we consider gradient descent on the population loss,
\begin{equation}
w_{t+1} = w_t - \eta \nabla F(w_t),
\label{eq:gd.pop.loss}
\end{equation}
and return $w_{t^*}:= \mathrm{argmin}_{0\leq t < T} F(w_t)$.  
The second setting is online SGD with samples $x_t\sim \calD$.  Here we compute unbiased estimates (conditional on $w_t$) of the population risk $F_t(w_t) := (1/2) (\sigma(w_t^\top x_t) - \sigma(v^\top x_t))^2$ and update the weights by
\begin{equation} 
w_{t+1} = w_t - \eta \nabla  F_t(w_t)
\label{eq:online.sgd}
\end{equation}
For SGD, we output $w_{t^*} = \mathrm{argmin}_{0\leq t < T} F_t(w_t)$.   

We summarize our results in the realizable case in Theorem \ref{theorem:gd.loss}.  
\begin{theorem}\label{theorem:gd.loss} 
Suppose $\pnorm{x}2\leq B$ a.s. and $\sigma$ is non-decreasing and $L$-Lipschitz.  Let $\eta \leq L^{-2}B^{-2}$  be the step size.
\begin{enumerate}[(a)]
    \item Let $\sigma$ satisfy Assumption \ref{assumption:activation.fcn}, and let $\gamma$ be the constant corresponding to $\rho=4B$.  For any initialization satisfying $\pnorm{w_0}2\leq 2$, if we run gradient descent on the population risk $T = \ceil{2\eps^{-1} L \eta^{-1} \gamma^{-1} \pnorm{w_0-v}2^2}$ iterations, then there exists $t<T$ such that $F(w_t)\leq \eps$.  For stochastic gradient descent, for any $\delta > 0$, running SGD for $\tilde T = 6T \log(1/\delta)$ guarantees there exists $w_t$, $t<T$, such that w.p. at least $1-\delta$, $F(w_t) \leq \eps$.
    
    \item Let $\sigma$ be ReLU and further assume that $\calD$ satisfies Assumption \ref{assumption:marginal.spread} for constants $\alpha,\beta>0$ and that $w_0=0$.  Let $\nu=\alpha^4 \beta /8\sqrt 2$.  If we run gradient descent on the population risk $T = \ceil{2\eps^{-1} L \eta^{-1} \nu^{-1} \pnorm{w_0-v}2^2}$ iterations, then there exists $t<T$ such that $F(w_t)\leq \eps$.  For stochastic gradient descent, for any $\delta > 0$, running SGD for $\tilde T = 6T \log(1/\delta)$ guarantees there exists $w_t$, $t<T$, such that w.p. at least $1-\delta$, $F(w_t) \leq \eps$.
\end{enumerate}
\end{theorem}

A few remarks on the above theorem: first, in comparison with our noisy neuron result in Theorem \ref{theorem:glm}, we are able to achieve $\opt+\eps=\eps$ population risk with sample complexity and runtime of order $\eps^{-1}$ rather than $\eps^{-2}$ using the same assumptions by invoking a martingale Bernstein inequality rather than Hoeffding.  Second, although Theorem \ref{theorem:gd.loss} requires the distribution to be bounded almost surely, we show in Section \ref{appendix:realizable.gd.pop.loss} below that for GD on the population loss, we can accomodate essentially any distribution with finite expected squared norm. 

\citet{yehudai20} used the marginal spread assumption to show that with probability 1/2, a single neuron in the realizable setting can be learned using gradient-based optimization with random initialization for Lipschitz activation functions satisfying $\inf_{0<z<\alpha} \sigma'(z) > 0$, where $\alpha$ is the same constant in Assumption \ref{assumption:marginal.spread}, and thus includes essentially all neural network activation functions like softplus, sigmoid, tanh, and ReLU.  Under the additional assumption of spherical symmetry, they showed that this can be improved to a high probability guarantee for the ReLU activation.  For gradient descent on the population risk, they proved linear convergence, i.e. a runtime of order $O(\log(1/\eps))$, while for SGD their runtime and sample complexity is of order $O(\eps^{-2} \log(1/\eps))$.  In comparison, our result for the non-ReLU activations requires only boundedness of the distributions and holds with high probability over random initializations, with runtime and sample complexity of order $O(\eps^{-1})$ for both gradient descent on the population risk and SGD.  Our results for ReLU use the same marginal spread assumption as Yehudai and Shamir, but our proof technique differs in that we do not require the angle $\theta(w_t,v)$ between the weights in the GD trajectory and the target neuron be decreasing.  As they pointed out, angle monotonicity fails to hold for the trajectory of gradient descent even when the distribution is a non-centered Gaussian, so that proofs based on angle monotonicity will not translate to more general distributions.  Indeed, our proofs in the agnostic and noisy teacher network setting use essentially the same proof technique as the realizable case without relying on angle monotonicity.  Instead, we show a type of inductive bias of gradient descent in the sense that if initialized at the origin, the angle between the target vector and the population risk minimizer cannot become larger than $\pi/2$, even in the agnostic setting.  

\subsection{Gradient descent on population loss}
\label{appendix:realizable.gd.pop.loss}
The key lemma for the proof is as follows.
\begin{lemma}\label{lemma:key}
Consider gradient descent on the population risk given in \eqref{eq:gd.pop.loss}.  Let $w_0$ be the initial point of gradient descent and assume $\pnorm{w_0}2 \leq 2$.  Suppose that $\calD$ satisfies $\E_x[\pnorm{x}2^2] \leq B^2$.  Let $\sigma$ be non-decreasing and $L$-Lipschitz.   Assume the step size satisfies $\eta \leq L^{-2} B^{-2}$. 
Then for any $T\in \N$, we have for all $t=0, \dots, T-1$, $\pnorm{w_t-v}2\leq \pnorm{w_0-v}2$, and
\[\pnorm{w_0-v}2^2 - \pnorm{w_T-v}2^2 \geq \eta L^{-1} \summm t 0 {T-1} \oF(w_t).\]
\end{lemma}

\begin{proof}
We begin with the identity, for $t < T$,
\begin{equation}
\pnorm{w_t - v}2^2 - \pnorm{w_{t+1}-v}2^2 = 2 \eta \ip{\nabla F(w_t)}{w_t-v} - \eta^2 \pnorm{\nabla F(w_t)}2^2.
\label{eq:identity.difference.in.distances}
\end{equation}
First, we have 

\begin{align}
\nonumber
\pnorm{\nabla F(w_t)}2 
&\leq \E_x\pnorm{(\sigma(w_t^\top x) - \sigma(v^\top x) ) \sigma'(w_t^\top x) x}2 \\
\nonumber
&\leq \sqrt{\E_x \l[ \sigma'(w_t^\top x) (\sigma(w_t^\top x) - \sigma(v^\top x))^2 \r]}  \sqrt{\E_x \sigma'(w_t^\top x) \pnorm x2^2 }\\
\nonumber
&\leq B \sqrt L \sqrt{ \E_x \l[ \sigma'(w_t^\top x) (\sigma(w_t^\top x) - \sigma(v^\top x))^2 \r] }.
\end{align}
The first inequality is by Jensen.  The second inequality uses that $\sigma'(z)\geq 0$ and H\"older, and the third inequality uses that $\sigma$ is $L$-Lipschitz and that $\E[\pnorm{x}2^2]\leq B^2$.  
We therefore have the gradient upper bound
\begin{equation}
\pnorm{\nabla F(w_t)}2^2 \leq 2B^2 L \oF(w_t).
\label{eq:gradient.upper.bound}
\end{equation}
For the inner product term of \eqref{eq:identity.difference.in.distances}, since $\sigma'(z)\geq 0$, we can use Fact \ref{fact:sigma.L.lipschitz} to get
\begin{align}
\ip{\nabla F(w_t)}{w_t-v}
&\geq  L^{-1} \E_x \l[ \l( \sigma(w_t^\top x) - \sigma(v^\top x) \r)^2 \sigma'(w_t^\top x) \r]
\label{eq:inner.product.lower.bound}
= 2 L^{-1} \oF(w_t).
\end{align}

Putting \eqref{eq:inner.product.lower.bound} and \eqref{eq:gradient.upper.bound} into \eqref{eq:identity.difference.in.distances}, we get
\begin{align}
\nonumber
\pnorm{w_t - v}2^2 - \pnorm{w_{t+1} - v}2^2 &\geq 4\eta L^{-1} \oF(w_t) - 2 \eta^2 B^2L \oF(w_t) \geq  2\eta L^{-1} \oF(w_t),
\label{eq:increment.wt}
\end{align}
where we have used $\eta \leq L^{-2} B^{-2}$.   Telescoping the above over $t < T$ gives 
\[ \pnorm{w_0-v}2^2 - \pnorm{w_T-v}2^2 \geq 2 \eta L^{-1} \summm t 0 {T-1} \oF(w_t).\]
Dividing each side by $\eta T$ shows the desired bound. 
\end{proof}

We will show that if $\sigma$ satisfies Assumption \ref{assumption:activation.fcn}, then Lemma \ref{lemma:key} allows for a population risk bound for essentially any distribution with $\E[\pnorm{x}2^2]\leq B^2$.  In particular, we consider distributions with finite expected norm squared and the possible types of tail bounds for the norm.  
\begin{assumption}
\begin{enumerate}[(a)]  
\item Bounded distributions: there exists $B>0$ such that $\pnorm{x}2 \leq B$ a.s.

\item Exponential tails: there exist $a_0, C_e> 0$ such that $\P(\pnorm{x}2^2 \geq a) \leq C_e \exp(-a)$ holds for all $a\geq a_0$.

\item Polynomial tails: there exist $a_0, C_p>0$ and $\beta >1$ such that $\P(\pnorm{x}2^2 \geq b) \leq C_p a^{- \beta}$ holds for all $a \geq a_0$. 
\end{enumerate}
\label{assumption:distribution}
\end{assumption}
If either (a), (b), or (c) holds, there exists $B>0$ such that $\E\pnorm{x}2^2 \leq B^2$.  One can verify that for (b), taking $B^2 = 2(a_0 \vee C_e)$ suffices, and for (c), $B^2 = 2 (a_0 \vee C_p^{1/\beta}/(1-\beta))$ suffices.  In fact, any distribution that satisfies $\E\pnorm x2^2 < \infty$ cannot have a tail bound of the form $\P(\pnorm x 2^2 \geq a) = \Omega(a^{-1})$, since in this case we would have
\[ \E\pnorm{x}2^2 = \int_0^\infty \P(\pnorm x2^2 > t) dt \geq C \int_{a_0}^\infty  t^{-1} dt = \infty.\]
So the polynomial tail assumption (c) is tight up to logarithmic factors for distributions with finite $\E \pnorm x2^2$.  

\begin{theorem}\label{theorem:appendix.nonrelu.gd.poprisk.loss}
Let $\E[\pnorm{x}2^2]\leq B^2$ and assume $\calD$ satisfies one of the conditions in Assumption \ref{assumption:distribution}.   Let $\sigma$ satisfy Assumption \ref{assumption:activation.fcn}.   
\begin{enumerate}[(a)]
\item Under Assumption \ref{assumption:distribution}a, let $\gamma$ be the constant corresponding to $\rho = 4B$ in Assumption \ref{assumption:activation.fcn}.  Running gradient descent for $T = \ceil{2\eps^{-1} L \eta^{-1} \gamma^{-1} \pnorm{w_0-v}2^2}$ guarantees there exists $t\in [T-1]$ such that $F(w_t) \leq \eps$. 

\item Under Assumption \ref{assumption:distribution}b, let $\gamma$ be the constant corresponding to $\rho = 4\sqrt{\log(18C_e / \eps)}$.  Running gradient descent for $T = \ceil{2\eps^{-1} L \eta^{-1} \gamma^{-1} \pnorm{w_0-v}2^2}$ guarantees there exists $t\in [T-1]$ such that $F(w_t) \leq \eps$.  

\item Under Assumption \ref{assumption:distribution}c, let $\gamma$ be the constant corresponding to $\rho = 4(18C_p/\eps(\beta-1))^{(1-\beta)/2}$.  Running gradient descent for $T = \ceil{2\eps^{-1} L \eta^{-1} \gamma^{-1} \pnorm{w_0-v}2^2}$ guarantees there exists $t\in [T-1]$ such that $F(w_t) \leq \eps$.  
\end{enumerate}
\end{theorem}
\begin{proof}
First, note that the conditions of Lemma \ref{lemma:key} hold, so that we have for all $t=0, \dots, T-1$, $\pnorm{w_t}2\leq 4$ and
\begin{equation}
\eta \summm t 0 {T-1} \oF(w_t) \leq L\pnorm{w_0-v}2^2 - L\pnorm{w_T-v}2^2.
\label{eq:key.decomp}
\end{equation}
By taking $T = \zeta^{-1} L \eps^{-1} \eta^{-1} \pnorm{w_0-v}2^2$ for arbitrary $\zeta > 0$, \eqref{eq:key.decomp} implies that there exists $t\in [T-1]$ such that
\begin{equation}
\oF(w_t) = \E \l[\l (\sigma(w_t^\top x) - \sigma(v^\top x)\r)^2 \sigma'(w_t^\top x) \r] \leq \f {L\pnorm{w_0-v}2^2}{\eta T} \leq \zeta \eps.
\label{eq:oF.bound}
\end{equation}
It therefore suffices to bound $F(w_t)$ in terms of the left hand side of \eqref{eq:oF.bound}.  We will do so by using the distributional assumptions given in Assumption \ref{assumption:distribution} and by choosing $\zeta$ appropriately.

We begin by noting that \eqref{eq:oF.bound} implies, for any $\rho>0$,
\begin{equation}
\E \l[ \l(\sigma(w_t^\top x) - \sigma(v^\top x)\r)^2 \sigma'(w_t^\top x) \ind(|w_t^\top x| \leq \rho) \r] \leq \zeta \eps.
\label{eq:ofsmallnorm.bound}
\end{equation}

For any $\rho > 0$, since $\pnorm{w_t}2\leq 4$, the inclusion
\begin{equation}
\Big \{ \pnorm x 2 \leq \rho/4 \Big \} \subset \Big \{ |w_t^\top x| \leq \rho \Big \},
\label{eq:normx.bounded}
\end{equation}
holds.  Under Assumption \ref{assumption:distribution}a, by taking $\rho = 4B$ and letting $\gamma$ be the corresponding constant from Assumption \ref{assumption:activation.fcn}, eqs. \eqref{eq:ofsmallnorm.bound} and \eqref{eq:normx.bounded} imply
\[ \gamma \E\l[ \l( \sigma(w_t^\top x) - \sigma(v^\top x)\r)^2\r] \leq \E \l[ \l(\sigma(w_t^\top x) - \sigma(v^\top x)\r)^2 \sigma'(w_t^\top x) \ind(\pnorm x 2 \leq \rho/4) \r] \leq \zeta \eps.\]
By taking $\zeta = \gamma/2$, this implies $F(w_t) \leq \eps$.

Under Assumption \ref{assumption:distribution}b, by taking $\rho = 4 \sqrt{a_0}$, we get
\begin{align}
\nonumber
\E \l[ \pnorm x2^2 \ind(\pnorm x2^2 >  \rho^2/4^2 ) \r] &= \int_{a_0}^\infty \P (\pnorm x2^2 > t) dt \\
&\leq C_e  \exp(-a_0).
\end{align}
Note that Assumption \ref{assumption:distribution}b holds if we take $a_0$ larger.  We can therefore let $a_0$ be large enough so that $a_0 \geq \log(18C_e/\eps)$, so that then
\begin{equation}
\E \l[ \pnorm x2^2 \ind(\pnorm x2^2 > \rho^2/4^2)\r] \leq \eps/18.
\label{eq:tail.Enorm}
\end{equation}

Similarly, under Assumption \ref{assumption:distribution}c, we can let $\gamma$ be the constant corresponding to $\rho = 4 \sqrt{a_0}$ and take $a_0 \geq (\eps(\beta-1)/18C_p)^{1/(1-\beta)}$ so that
\begin{align}
\nonumber
\E \l[ \pnorm x2^2 \ind(\pnorm x2^2 >  \rho^2/4^2 ) \r] &= \int_{a_0}^\infty \P (\pnorm x2^2 > t) dt \\
\nonumber
&\leq C_p \f{a_0^{1-\beta}}{\beta-1}\\
\nonumber
&\leq \eps/18.
\end{align}
and so \eqref{eq:tail.Enorm} holds as well under Assumption \ref{assumption:distribution}c.  
We can therefore bound
\begin{align}
\nonumber
\E \l[  \l( \sigma(w_t^\top x) - \sigma(v^\top x)\r)^2 \ind(\pnorm x2^2 > \rho^2 /4^2)\r] &\leq \E\l[ \pnorm{w_t-v}2^2 \pnorm x2^2 \ind(\pnorm x2^2 > \rho^2/4^2 )\r] \\
\nonumber
&\leq \pnorm{w_0-v}2^2 \E \l[ \pnorm x2^2 \ind(\pnorm x2^2 > \rho^2/4^2) \r]\\
\nonumber
&\leq \pnorm{w_0-v}2^2 \eps/18\\
\label{eq:poly.tail.bignorm.bound}
&\leq \eps/2.
\end{align}
The first inequality uses that $\sigma$ is 1-Lipschitz and Cauchy--Schwarz.  The second inequality uses \eqref{eq:key.decomp}.  The third inequality uses \eqref{eq:tail.Enorm}.  The final inequality uses that $\pnorm{w_0-v}2\leq \pnorm {w_0}2 + \pnorm v2\leq 3$. 

We can then guarantee
\begin{align*}
2\gamma F(w_t) &= \gamma \E\l[ \l( \sigma(w_t^\top x) - \sigma(v^\top x)\r)^2\r]\\\
&= \E\l[ \l( \sigma(w_t^\top x) - \sigma(v^\top x)\r)^2 \gamma \ind(|w_t^\top x| \leq \rho)\r] \\
&\quad+ \gamma \E\l[ \l( \sigma(w_t^\top x) - \sigma(v^\top x)\r)^2 \ind(|w_t^\top x| > \rho)\r]\\
&\leq \E \l[\l( \sigma(w_t^\top x) - \sigma(v^\top x)\r)^2 \sigma'(w_t^\top x) \ind(|w_t^\top x| \leq \rho)\r]  \\
&\quad + \gamma \E \l[ \l( \sigma(w_t^\top x) - \sigma(v^\top x) \r)^2 \ind(\pnorm x2^2 > \rho^2/4^2) \r]\\
&\leq  \zeta \eps + \gamma \eps/2\\
&\leq \gamma \eps.
\end{align*}
The first inequality follows since Assumption \ref{assumption:activation.fcn} implies $\sigma'(z) \ind(|z| \leq \rho) \geq \gamma \ind(|z| \leq \rho)$ and by \eqref{eq:normx.bounded}.  The second inequality uses \eqref{eq:ofsmallnorm.bound} and \eqref{eq:poly.tail.bignorm.bound}.   The final inequality takes $\zeta = \gamma/2$. 
\end{proof}

\begin{remark}
The precise runtime guarantee in Theorem \ref{theorem:gd.loss} will depend upon the activation function and tail distribution.  The worst-case activation functions (like the sigmoid) can have $\gamma \sim \exp(-\rho)$, and so if one only has polynomial tails, the runtime can be exponential in $\eps^{-1}$ in this case.  If the distribution of $\pnorm{x}2^2$ has exponential tails, as is the case if the components of $x$ are sub-Gaussian, runtime will be polynomial in $\eps^{-1}$. On the other hand, if the $\gamma$ in Assumption \ref{assumption:activation.fcn} is a fixed constant independent of $\rho$ (as it is for the leaky ReLU), any of the tail bounds under consideration will have runtime of order $\eps^{-1}$.  
\end{remark}

\subsection{Stochastic gradient descent proofs}\label{appendix:realizable.sgd}
We consider the online version of stochastic gradient descent, where we sample independent samples $x_t\sim \calD$ at each step and compute stochastic gradient updates $g_t$, such that
\[ g_t = \l( \sigma(w_t^\top x_t) - \sigma(v^\top x_t) \r) \sigma'(w_t^\top x_t) x_t,\quad w_{t+1} = w_t - \eta g_t.\]
As in the gradient descent case, we have a key lemma that relates the distance of the weights at iteration $t$ from the optimal $v$ with the distance from initialization and the cumulative loss.
\begin{lemma}\label{lemma:key.sgd}
Assume that $\sigma$ is non-decreasing and $L$-Lipschitz, and that $\calD$ satisfies $\pnorm{x}2 \leq B$ a.s.  Assume the initialization satisfies $\pnorm{w_0}2 \leq 2$.  Let $T\in \N$ and run stochastic gradient descent for $T-1$ iterations at a fixed learning rate $\eta$ satisfying $\eta \leq L^{-2} B^{-2}$.  Then with probability one over $\calD$, we have $\pnorm{w_{t+1}-v}2\leq \pnorm{w_t-v}2$ for all $t< T$, and
\[ \pnorm{w_0 - v}2^2 - \pnorm{w_T-v}2^2 \geq 2 \eta L^{-1} \summm t 0 {T-1} \oF_t,\]
where $\oF_t := \f 1 2 \l( \sigma(w_t^\top x_t) - \sigma(v^\top x_t)\r)^2 \sigma'(w_t^\top x_t)$.  
\end{lemma}
\begin{proof}
We begin with the decomposition
\begin{equation}
\pnorm{w_t - v}2^2 - \pnorm{w_{t+1}-v}2^2 = 2 \eta \ip {g_t}{w_t - v} - \eta^2 \pnorm{g_t}2^2.
\label{eq:sgd.decomp}
\end{equation}
By Assumption \ref{assumption:activation.fcn}, since $\pnorm x 2 \leq B$ a.s. it holds with probability one that
\begin{align}
\pnorm{g_t}2^2 &= \pnorm{\l( \sigma(w_t^\top x_t) - \sigma(v^\top x_t)\r) \sigma'(w_t^\top x_t) x_t }2^2 \leq 2 L B^2 \oF_t.
\label{eq:sgd.grad.ub}
\end{align}
By Fact \ref{fact:sigma.L.lipschitz}, since $\sigma'(z)\geq 0$, we have with probability one,
\begin{align}
\nonumber
\ip{g_t}{w_t-v} &= \l( \sigma(w_t^\top x_t) - \sigma(v^\top x_t) \r) \sigma'(w_t^\top x_t) (w_t^\top x_t - v^\top x_t)\\
\nonumber
&\geq L^{-1} \l( \sigma(w_t^\top x_t) - \sigma(v^\top x_t) \r)^2 \sigma'(w_t^\top x_t)\\
&= 2 L^{-1} \oF_t.
\label{eq:sgd.ip.lb}
\end{align}
Putting \eqref{eq:sgd.grad.ub} and \eqref{eq:sgd.ip.lb} into \eqref{eq:sgd.decomp}, we get
\begin{align*}
\pnorm{w_t-v}2^2 - \pnorm{w_{t+1}-v}2^2 &\geq 4 \eta L^{-1} \oF_t - 2 \eta^2 L B^2 \oF_t\\
&\geq 2 \eta  L^{-1} \oF_t,
\end{align*}
by taking $\eta \leq L^{-2} B^{-2}$.  Telescoping over $t<T$ gives the desired bound.

\end{proof}

We now want to translate the bound on the empirical error to that of the true error. For this we use a martingale Bernstein inequality of \citet{beygelzimer}.  A similar analysis of SGD was used by~\citet{jitelgarsky} for a one-hidden-layer ReLU network. 
\begin{lemma}[\citet{beygelzimer}, Theorem 1]\label{lemma:beygelzimer}
Let $\{Y_t\}$ be a martingale adapted to the filtration $\calF_t$, and let $Y_0=0$.  Let $\{D_t\}$ be the corresponding martingale difference sequence.  Define the sequence of conditional variance
\[ V_t := \summ k t \E[D_k^2 | \calF_{k-1}],\]
and assume that $D_t \leq R$ almost surely.  Then for any $\delta \in (0,1)$, with probability greater than $1-\delta$,
\[ Y_t \leq R \log(1/\delta) + (e-2)V_t/R.\]
\end{lemma}

\begin{lemma}\label{lemma:sgd.bernstein.monotone}
Suppose that $\pnorm{x}2\leq B$ a.s., and let $\sigma$ be non-decreasing and $L$-Lipschitz.  Assume that the trajectory of SGD satisfies $\pnorm{w_t-v}2 \leq \pnorm{w_0-v}2$ for all $t$ a.s.  We then have with probability at least $1-\delta$,
\[ \f 1 T \summm t 0 {T-1} \oF(w_t) \leq \f 4 T \summm t 0 {T-1} \oF_t + \f 2 T B^2 L^3 \pnorm{w_0-v}2^2 \log(1/\delta).\]
\end{lemma}
\begin{proof}
Let $\calF_t = \sigma(x_0, \dots, x_t)$ be the $\sigma$-algebra generated by the first $t+1$ draws from $\calD$.  Then the random variable $G_t := \summm \tau 0 t (\oF(w_\tau) - \oF_\tau)$ is a martingale with respect to the filtration $\calF_{t}$ with martingale difference sequence $D_t := \oF(w_t) - \oF_t$.  We need bounds on $D_t$ and on $\E[D_t^2 | \calF_{t-1}]$ in order to apply Lemma \ref{lemma:beygelzimer}. 

Since $\sigma$ is $L$-Lipschitz and $\pnorm{x}2\leq B$ a.s., with probability one we have
\begin{equation}
D_t \leq \oF(w_t) \leq \f 1 2 L^3 B^2 \pnorm{w_t-v}2^2 \leq \f 1 2 L^3 B^2 \pnorm{w_0-v}2^2.
\label{eq:sgd.Fwt.ub}
\end{equation}
The last inequality uses the assumption that $\pnorm{w_t-v}2\leq \pnorm{w_0-v}2$ a.s.  
Similarly, 
\begin{align}
\nonumber
\E[\oF_t^2 | \calF_{t-1}] &= \f 1 4 \E\l[ \l( \sigma(w_t^\top x_t) - \sigma(v^\top x_t)\r)^4 \sigma'(w_t^\top x_t)^2| \calF_{t-1} \r] \\
\nonumber
\nonumber
&\leq \f 14 L^3B^2 \pnorm{w_t-v}2^2 \E_x \l[ \l( \sigma(w_t x_t) - \sigma(v^\top x_t)\r)^2  \sigma'(w_t^\top x_t)| \calF_{t-1}\r]\\
\label{eq:sgd:squaredincrement:Ft}
&\leq \f 1 2 L^3 B^2 \pnorm{w_0-v}2^2 \oF(w_t).
\end{align}
In the first inequality, we have used $\pnorm{x}2^2\leq B^2$ a.s. and $L$-Lipschitzness of $\sigma$.  For the second, we use the assumption that $\pnorm{w_t-v}2\leq \pnorm{w_0-v}2$ together with the fact that $\E_x[\oF_t | \calF_{t-1}] = \oF(w_t)$.  We then can use \eqref{eq:sgd:squaredincrement:Ft} to bound the squared increments,
\begin{align}
\nonumber
\E[D_t^2 | \calF_{t-1}] &= \oF(w_t)^2 - 2 \oF(w_t) \E[\oF_t | \calF_{t-1}] + \E [\oF_t^2 | \calF_{t-1}]\\
\nonumber
&= -\oF(w_t)^2 + \E[\oF_t^2 | \calF_{t-1}] \\
\label{eq:sgd:squaredincrement:Dt}
&\leq \f 1 2 L^3 B^2 \pnorm{w_0-v}2^2 \oF(w_t).
\end{align}
This allows for us to bound
\[ V_T := \summm t 0 {T-1} \E[D_t^2 | \calF_{t-1}] \leq \f 1 2 B^2 L^3 \pnorm{w_0-v}2^2 \summm t0{T-1} \oF(w_t). \]
Since $D_t \leq \oF(w_t)\leq (1/2)L^3 B^2 \pnorm{w_0-v}2^2$ a.s. by \eqref{eq:sgd.Fwt.ub}, Lemma \ref{lemma:beygelzimer} implies that with probability at least $1-\delta$, we have
\begin{align*}
\summm t 0 {T-1} (\oF(w_t) - \oF_t) \leq (\exp(1)-2) \summm t 0 {T-1} \oF(w_t) + \f 1 2 L^3 B^2 \pnorm{w_0-v}2^2 \log(1/\delta),
\end{align*}
and using that $(1-\exp(1)+2)^{-1}\leq 4$, we divide each side by $T$ and get
\begin{equation}
\f 1 T \summm t 0{T-1}  \oF(w_t) \leq \f 4 T \summm t 0 {T-1} \oF_t + \f 2 T L^3 B^2 \pnorm{w_0-v}2^2 \log(1/\delta).
\end{equation}
\end{proof}

With the above in hand, we can prove Theorem \ref{theorem:gd.loss} in the SGD setting.

\begin{proof}[Proof of Theorem \ref{theorem:gd.loss}, SGD]
By the assumptions in the theorem, Lemma \ref{lemma:key.sgd}  holds, so that we have for any $t=0,\dots, T-1$, $\pnorm{w_t}2 \leq 4$ and
\begin{equation}
\pnorm{w_t-v}2^2 + 2 \eta L^{-1} \summm \tau 0 {t-1} \oF_\tau\leq \pnorm{w_0-v}2^2.
\label{eq:sgd.key}
\end{equation}
This shows that $\pnorm{w_t-v}2 \leq \pnorm{w_0-v}2$ holds for all $t=0, \dots, T-1$ a.s., allowing for the application of Lemma \ref{lemma:sgd.bernstein.monotone} to get
\begin{equation}
\f 1 T \summm t 0 {T-1} \oF(w_t) \leq \f 4 T \summ t T \oF_t + \f 2 T L^3 B^2 \pnorm{w_0-v}2^2 \log(1/\delta).
\label{eq:sgd.nonrelu.final}
\end{equation}
Dividing both sides of \eqref{eq:sgd.key} by $\eta T L^{-1}$ yields 
\[ \min_{t <T} \oF(w_t) \leq \f 1 T \summm t 0 {T-1} \oF(w_t) \leq \f{ L\pnorm{w_0-v}2^2}{\eta T} + \f 2 T L^3 B^2 \pnorm{w_0-v}2^2 \log(1/\delta).\]
For arbitrary $\zeta >0$, taking $T = \ceil{2\eps^{-1} \zeta^{-1} \eta^{-1} L^3 B^2 \pnorm{w_0-v}2^2 \log(1/\delta)}$ shows there exists $T$ such that $\oF(w_t)\leq \zeta\eps$.  When $\sigma$ satisfies Assumption \ref{assumption:activation.fcn}, since $\pnorm{w_t}2\leq 4$ for all $t$, it holds that $\oF(w_t) \geq \gamma F(w_t)$, so that $\zeta = \gamma$ furnishes the desired bound.  

When $\sigma$ is ReLU and $\calD$ satisfies Assumption \ref{assumption:marginal.spread}, we note that Lemma \ref{lemma:key.sgd} implies $\pnorm{w_t-v}2 \leq \pnorm{w_0-v}2$ a.s.  Thus taking $\zeta = \alpha^4 \beta / 8 \sqrt 2$ and using Lemma \ref{lemma:relu.of.implies.f} completes the proof. 
\end{proof}

 \section{Remaining Proofs}
\label{appendix:simple.proofs}
\begin{proof}[Proof of Lemma \ref{lemma:F(v).opt.concentration}]
Since $\sigma$ is non-decreasing, $|\sigma(v^\top x) + y| \leq |\sigma(B_X)| + B_Y$.  In particular, each summand defining $\hat F(v)$ is a random variable with absolute value at most $a = (|\sigma(B_X)| + B_Y)^2$.  As $\E[\hat F(v)] = F(v)= \opt$, Hoeffding's inequality implies the lemma.
\end{proof}

\begin{proof}[Proof of Lemma \ref{lemma:rademacher.complexity}]
 The bound $\mathfrak{R}_S(\calG) \leq 2 \max_i \pnorm{x_i}2 / \sqrt n$ follows since $\pnorm{w}2\leq 2$ holds on $\calG$ with standard results  Rademacher complexity theory (e.g. Sec. 26.2 of~\cite{shalevschwartz}); this shows $\mathfrak{R}(\calG) \leq 2 B_X/\sqrt n$.  Using the contraction property of the Rademacher complexity, this implies $\mathfrak{R} (\sigma \circ \calG) \leq 2B_XL/\sqrt n$.  Finally, note that if $\pnorm{w-v}2 \leq 1$ and $\pnorm x2 \leq B_X$, we have
 \begin{align}
     \norm{\nabla \ell (w; x)} = \norm{ \l( \sigma(w^\top x) - \sigma(v^\top x) \r) \sigma'(w^\top x) x } \leq L^2 \norm{w-v} \norm{x} \leq L^2 B_X.
 \end{align}
 In particular, $\ell$ is $L^2 B_X$ Lipschitz.  The result follows.
 \end{proof}
 

\bibliography{references}

\end{document}